\pdfoutput=1

\documentclass{article}

\usepackage{microtype}
\usepackage{graphicx}
\usepackage{subcaption}
\usepackage{booktabs} 

\usepackage[accepted]{icml2026}
\usepackage{multirow}
\usepackage[table]{xcolor}
\usepackage{makecell}
\definecolor{Gray}{gray}{0.93}

\renewcommand{\algorithmicrequire}{\textbf{Input:}}
\renewcommand{\algorithmicensure}{\textbf{Output:}}

\definecolor{mydarkred}{rgb}{0.6,0,0}
\definecolor{mydarkgreen}{rgb}{0,0.6,0}
\PassOptionsToPackage{options}{natbib}
\setcitestyle{authoryear,round,citesep={;},aysep={,},yysep={;}}
\usepackage[colorlinks,linkcolor=mydarkred,citecolor=mydarkgreen]{hyperref}

\hypersetup{
  bookmarks=true,
  bookmarksnumbered=true,
  bookmarksopen=true,
  bookmarksopenlevel=2,
  pdfstartview=FitH
}

\usepackage{amsmath}
\usepackage{amssymb}
\usepackage{mathtools}
\usepackage{amsthm}

\renewcommand{\hat}{\widehat}

\usepackage[capitalize,noabbrev]{cleveref}

\usepackage{tikz}
\usetikzlibrary{positioning,shapes,arrows.meta}

\theoremstyle{plain}
\newtheorem{theorem}{Theorem}[section]
\newtheorem{proposition}[theorem]{Proposition}

\theoremstyle{definition}
\newtheorem{definition}[theorem]{Definition}
\newtheorem{assumption}[theorem]{Assumption}
\theoremstyle{remark}
\newtheorem{remark}{Remark}

\makeatletter
\def\th@remark{%
  \thm@headfont{\bfseries\itshape}%
  \normalfont 
  \thm@preskip\topsep \divide\thm@preskip\tw@
  \thm@postskip\thm@preskip
}
\makeatother

\icmltitlerunning{Conditional Performance Guarantee for Large Reasoning Models}

\begin{document}

\twocolumn[
\icmltitle{Conditional Performance Guarantee for Large Reasoning Models}



\icmlsetsymbol{equal}{*}

\begin{icmlauthorlist}
\icmlauthor{Jianguo Huang}{ntu,equal}
\icmlauthor{Hao Zeng}{sustech,equal}
\icmlauthor{Bingyi Jing}{cuhk,sustech}
\icmlauthor{Hongxin Wei}{sustech}
\icmlauthor{Bo An}{ntu}
\end{icmlauthorlist}

\icmlaffiliation{ntu}{Nanyang Technological University, Singapore}
\icmlaffiliation{sustech}{Southern University of Science and Technology, China}
\icmlaffiliation{cuhk}{The Chinese University of Hong Kong, Shenzhen, China}

\icmlcorrespondingauthor{Hongxin Wei}{weihx@sustech.edu.cn}

\icmlkeywords{Machine Learning, ICML, PAC Reasoning, Efficient Reasoning}

\vskip 0.3in
]



\printAffiliationsAndNotice{\icmlEqualContribution}  

\begin{abstract}
Large reasoning models have shown strong performance through extended chain-of-thought reasoning, yet their computational cost remains significant.
Probably approximately correct (PAC) reasoning provides statistical guarantees for efficient reasoning by adaptively switching between thinking and non-thinking models, but the guarantee holds only in the marginal case and does not provide exact conditional coverage.
We propose G-PAC reasoning, a practical framework that provides PAC-style guarantees at the group level by partitioning the input space.
We develop two instantiations: Group PAC (G-PAC) reasoning for known group structures and Clustered PAC (C-PAC) reasoning for unknown groupings.
We prove that both G-PAC and C-PAC achieve group-conditional risk control, and that grouping can strictly improve efficiency over marginal PAC reasoning in heterogeneous settings.
Our experiments on diverse reasoning benchmarks demonstrate that G-PAC and C-PAC successfully achieve group-conditional risk control while maintaining substantial computational savings.
\end{abstract}


\section{Introduction}

Large language models (LLMs) demonstrate strong reasoning capabilities through extended chain-of-thought generation~\citep{deepseek-ai2025deepseekr1,yang2025qwen3}.
However, these improvements come at the cost of overthinking~\citep{sui2025stop,yue2025dont,aggarwal2025optimalthinkingbench}, resulting in increased latency and computational overhead when deploying LLMs at scale. 
This shows the importance of balancing reasoning performance and computational efficiency.
Recent work on efficient reasoning~\citep{zeng2025pac}, called \textit{PAC reasoning}, addresses this challenge by using a statistical framework based on adaptive model selection. Specifically, the framework maintains two models: a thinking model that performs full reasoning and a non-thinking model that produces fast but potentially less accurate responses. PAC reasoning constructs a composite system that adaptively routes inputs to either model based on uncertainty estimates, while providing statistical guarantees on the expected performance loss.

While PAC reasoning provides a statistical solution for efficient reasoning, its guarantees are marginal: the performance bound holds on average over the entire input distribution, allowing arbitrarily large errors on subpopulations.
Such marginal guarantees are insufficient in many high-stakes applications.
For instance, in medical diagnosis, practitioners require reliable guarantees for specific diseases, rather than guarantees that hold only on average across the population.
This raises a fundamental question: 
\begin{center}
\emph{Can we provide performance loss guarantees conditional on each group?}
\end{center}

We propose a novel statistical method, termed \textit{Group PAC (G-PAC) reasoning}.
Instead of seeking impossible per-input guarantees~\citep{zeng2025note}, G-PAC partitions the input space into groups and provides PAC guarantees at the group level, enabling explicit control of performance loss within each group.
When group information is not available a priori, we augment G-PAC with a clustering step, referred to as \textit{Clustered PAC (C-PAC) reasoning}, which learns the grouping from calibration data using uncertainty scores.
We establish theoretical guarantees for G-PAC under both known and learned grouping settings.
For known groupings, we prove group-wise PAC guarantees together with finite-sample risk bounds.
For learned groupings, we characterize the oracle optimal partition and show that grouping strictly improves efficiency under heterogeneity.
Finally, we provide coverage guarantees for the learned partitions and show that the resulting coverage gap vanishes asymptotically.

We conduct comprehensive experiments in Section~\ref{sec:experiments} evaluating G-PAC reasoning across diverse reasoning benchmarks, including MATH-500~\citep{lightman2023lets}, 
ZebraLogic~\citep{lin2025zebralogic}, and GPQA~\citep{rein2024gpqa}. 
Our results show that G-PAC reasoning effectively controls group-conditional loss while substantially reducing inference cost. 
For instance, on MATH-500 with tolerance $\epsilon = 0.05$ under the logits-based score, 
G-PAC achieves zero group-conditional error gap, whereas vanilla PAC fails to control group-wise risk.
Moreover, even when group partitions are not available a priori, 
G-PAC remains effective by learning group structure directly from data.

Our main contributions are as follows.
\begin{itemize}
\item We formalize {group-conditional PAC efficiency}, a practical relaxation that provides {statistical guarantees at the group level} while enabling fine-grained control over the accuracy-efficiency trade-off. 
\item We design a practical PAC reasoning method, \emph{G-PAC reasoning}, together with a clustering-based extension, \emph{C-PAC reasoning}, for settings where group information is unavailable, and show that both satisfy group-conditional PAC efficiency in practice.

\item We establish theoretical guarantees showing that group-conditional risk can be provably controlled under both known and learned groupings. Building on this result, we show that grouping strictly improves efficiency under heterogeneity, characterize the oracle optimal partition, and prove that the coverage gap vanishes as the learned partition approaches the oracle.
\end{itemize}

\subsection{Related work}
\label{sec:related-work}

\textbf{Efficient reasoning for large language models.}
LLMs~\citep{deepseek-ai2025deepseekr1,yang2025qwen3} have shown strong performance in complex reasoning tasks, often at a high computational cost, i.e., \emph{overthinking}~\citep{chen2025reasoning,sui2025stop}.
To address this issue, recent works propose efficient reasoning strategies such as early exit~\citep{yang2025dynamic,jiang2025flashthink} and adaptive switching of a single LLM between thinking and non-thinking modes~\citep{cheng2025think,chung2025thinker,fang2025thinkless,ma2025reasoning,xiao2025fastslow,yong2025think}.
Model routing and cascading approaches~\citep{dekoninck2025unified} select among multiple LLMs based on query difficulty.
Despite their empirical effectiveness, these techniques lack theoretical guarantees on the resulting performance degradation.
PAC reasoning~\citep{zeng2025pac} addresses this limitation by providing statistical guarantees that characterize the trade-off between computational efficiency and accuracy.
However, these guarantees are \emph{marginal}, holding only in expectation over the entire input distribution.
We move beyond marginal analysis and propose \emph{group-conditional PAC-efficient reasoning}, which enables fine-grained control of performance loss across different groups of inputs.

\textbf{PAC learning and distribution-free risk control.}
PAC learning~\citep{valiant1984theory} establishes how algorithms can generalize from training data with probabilistic guarantees.
The Learn-then-Test (LTT) framework~\citep{angelopoulos2025learn,bates2021distributionfree} provides a modern, distribution-free approach to PAC-style guarantees, enabling explicit specification of error tolerances.
Conformal risk control~\citep{angelopoulos2025conformal} extends this framework to control the expected value of monotone loss functions.
Despite their generality, these methods primarily yield \emph{marginal} guarantees that hold only on average over the input distribution, and may fail to control performance loss for particular subpopulations.
This limitation naturally raises the question of whether we can achieve conditional guarantees that hold for each input group.
Our group-conditional PAC-efficient reasoning framework addresses this challenge by providing guarantees at the group level, so it enables clear computational efficiency gains across subpopulations.

\section{PAC reasoning}
\label{sec:pac-efficient}
We introduce \textit{PAC reasoning}~\cite{zeng2025pac}, which only achieves marginal PAC efficiency. 
However, this marginal guarantee may lead to uncontrolled risk on individual input, motivating the need for a conditional PAC efficiency.

\textit{PAC reasoning} formalizes the LLM's selective thinking with a PAC guarantee on the performance risk.
Let $x\sim\mathcal{P}$ denote an input prompt drawn from a data distribution $\mathcal{P}$, $f$ a thinking LLM that generates extended reasoning chains, and $\tilde{f}$ a non-thinking LLM that generates direct answers without thinking parts.
Then, PAC reasoning constructs a composite model $\hat{f}$ that routes inputs between $f$ and $\tilde{f}$ using an uncertainty score $U(x)\in[0,1]$ and a calibrated threshold $u$. Formally, the routing rule is given by:
\begin{equation*}
\hat{f}(x) = \begin{cases}
\tilde{f}(x) & \text{if } U(x) \leq u, \\
 f(x) & \text{else},
\end{cases}
\end{equation*}
where $U(x)$ is a score to quantify the uncertainty of the non-thinking LLM on the input $x$. 
Then, the performance loss of the composite model $\hat{f}$ is defined as:
\begin{equation*}
R(\hat{f}) = \mathbb{E}_{x \sim \mathcal{P}}[\ell(\hat{f}(x), f(x))],
\end{equation*}
where $\ell$ is a loss function that measures the performance degradation of $\hat{f}$ relative to the thinking LLM $f$, such as the 0-1 loss for verifiable answers. 
\begin{remark}
    Notably, the loss is measured \textit{relative to the thinking model's output} (rather than the ground truth), as our goal is to bound the additional degradation induced by routing away from the thinking model.
\end{remark}

To control performance loss with a theoretical guarantee, we introduce the definition of PAC-efficient:
\begin{definition}[PAC-efficient]
\label{def:pac-efficient}
Given an error tolerance $\varepsilon > 0$ and a confidence level $\alpha \in (0, 1)$, a composite model $\hat{f}$ is $(\varepsilon, \alpha)$-PAC-efficient if
\begin{equation*}
\mathbb{P}(R(\hat{f}) \leq \varepsilon) \geq 1 - \alpha.
\end{equation*}
\end{definition}

Definition~\ref{def:pac-efficient} provides a marginal guarantee that controls the expected performance loss averaged over the entire input distribution.
PAC reasoning offers a model-free procedure to achieve such PAC-efficient bounds.
However, the marginal guarantee does not impose any constraint on the performance loss within specific input subpopulations.
As a result, the composite model may satisfy the global PAC guarantee while still exhibiting substantially larger losses on certain groups, potentially violating the target tolerance at the group level.
This limitation motivates the need for a finer-grained notion of PAC efficiency that provides conditional guarantees beyond the marginal setting.

\paragraph{Impossible full conditional PAC efficiency}
A natural strengthening is to require the guarantee to hold for each input $x$ separately.
For a specific input $x$, define the conditional risk $R(\hat{f} \mid x) = \mathbb{E}_{\mathcal{D}_{\text{cal}}}[\ell(\hat{f}(x), f(x)) \mid x]$, where the expectation is taken over the randomness of the calibration set.
A composite model $\hat{f}$ is $(\varepsilon, \alpha)$-pointwise PAC efficient if for every $x \in \mathcal{X}$,
\begin{equation}
\label{eq:ccpac}
\mathbb{P}(R(\hat{f} \mid x) \leq \varepsilon) \geq 1 - \alpha.
\end{equation}
Unfortunately, achieving such per-instance guarantees is impossible without sacrificing efficiency~\citep{zeng2025note}: \textit{the only way to satisfy pointwise PAC efficiency is to use the expert model for \textbf{almost every} $x \in \mathcal{X}$.}

\section{Group-conditional PAC-efficient reasoning}
\label{sec:group-conditional-pac}

To address the limitations of standard PAC reasoning, we introduce the \textit{group-conditional PAC efficiency} and propose methods termed \textit{Group PAC (G-PAC) reasoning} and \textit{Clustered PAC (C-PAC) reasoning}, to control performance loss within known and unknown groups of partitions.

\begin{figure*}[t]
\centering
\begin{tikzpicture}[
  label/.style={font=\small, align=center}
]

\def\PtsGreen{{-1.0/0.9},{-0.6/0.5},{-0.4/0.8},{-0.8/0.4},{-0.3/1.0}}
\def\PtsRed{{0.5/0.1},{0.9/-0.2},{0.6/-0.4},{1.0/0.0},{0.4/0.3}}
\def\PtsYellow{{-0.6/-0.8},{-0.3/-1.0},{-0.2/-0.7},{-0.5/-0.9},{0.0/-1.1}}

\begin{scope}[shift={(-5,0)}]
  \draw[dashed, thick, gray!60] (0,0) circle (1.8cm);
  \foreach \x/\y in \PtsGreen {
    \fill[green!70!black] (\x,\y) circle (0.15);
  }
  \foreach \x/\y in \PtsRed {
    \fill[red!70!black] (\x,\y) circle (0.15);
  }
  \foreach \x/\y in \PtsYellow {
    \fill[yellow!80!orange] (\x,\y) circle (0.15);
  }
  \node[label] at (0,-2.8) {Feasible but less efficient};
  \node[font=\normalsize] at (0,2.5) {PAC reasoning};
\end{scope}

\begin{scope}[shift={(0,0)}]
  \draw[dashed, thick, gray!60] (-0.6,0.6) ellipse (1.1cm and 0.9cm);
  \draw[dashed, thick, gray!60] (0.7,-0.1) ellipse (1.0cm and 0.8cm);
  \draw[dashed, thick, gray!60] (-0.4,-0.9) ellipse (0.9cm and 0.6cm);
  \foreach \x/\y in \PtsGreen {
    \fill[green!70!black] (\x,\y) circle (0.15);
  }
  \foreach \x/\y in \PtsRed {
    \fill[red!70!black] (\x,\y) circle (0.15);
  }
  \foreach \x/\y in \PtsYellow {
    \fill[yellow!80!orange] (\x,\y) circle (0.15);
  }
  \node[font=\scriptsize, text=green!70!black] at (-1.4,0.8) {$\hat{u}_1$};
  \node[font=\scriptsize, text=red!70!black] at (1.3,0.2) {$\hat{u}_2$};
  \node[font=\scriptsize, text=yellow!80!orange] at (-0.1,-1.6) {$\hat{u}_3$};
  \node[label, font=\normalsize, text=cyan!70!blue] at (0,-2.8) {\textbf{Feasible and group-wise efficient}};
  \node[font=\normalsize] at (0,2.5) {G-PAC reasoning};
\end{scope}

\begin{scope}[shift={(5,0)}]
  \foreach \x/\y in \PtsGreen {
    \draw[dashed, gray!50] (\x,\y) circle (0.28);
    \fill[green!70!black] (\x,\y) circle (0.15);
  }
  \foreach \x/\y in \PtsRed {
    \draw[dashed, gray!50] (\x,\y) circle (0.28);
    \fill[red!70!black] (\x,\y) circle (0.15);
  }
  \foreach \x/\y in \PtsYellow {
    \draw[dashed, gray!50] (\x,\y) circle (0.28);
    \fill[yellow!80!orange] (\x,\y) circle (0.15);
  }
  \node[label, text=red!80!black] at (0,-2.8) {\textbf{Most efficient but impossible}};
  \node[font=\normalsize] at (0,2.5) {Conditional PAC reasoning};
\end{scope}

\draw[<->, thick, gray] (-5,-2.0) -- (5,-2.0);

\end{tikzpicture}
\caption{\textbf{Trade-off between feasibility and efficiency}.
Left: PAC reasoning uses a single threshold for all inputs, which is feasible but not the most efficient.
Right: conditional PAC reasoning requires per-input guarantees, which is fully efficient but impossible.
Middle: G-PAC reasoning balances this trade-off by grouping similar inputs and calibrating dynamic thresholds $\hat{u}_j$ for each group, making it both feasible and group-conditional efficient.}
\label{fig:clustered-pac}
\end{figure*}
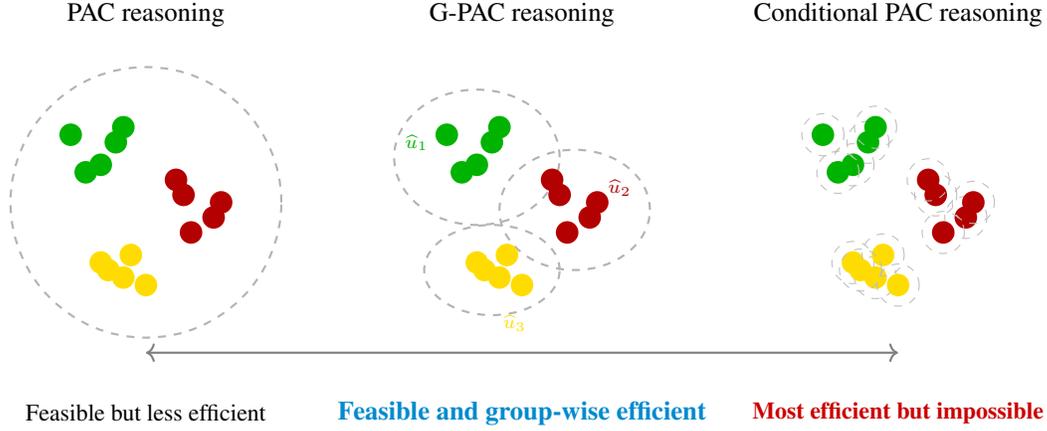

\subsection{Group PAC (G-PAC) reasoning}
\label{sec:gpac}

Let $\mathcal{G} = \{G_1, \ldots, G_k\}$ be a partition of the input space $\mathcal{X}$ with $k$ groups.
For each group $G_j$, we define the group-conditional risk function:
\begin{equation*}
R(\hat{f} \mid G_j) = \mathbb{E}_{x \sim \mathcal{P}}[\ell(\hat{f}(x), f(x)) \mid x \in G_j].
\end{equation*}
We now introduce a group-conditional PAC efficiency:
\begin{definition}[Group-conditional PAC efficiency]
\label{def:group-conditional-pac}
Given an error tolerance $\varepsilon > 0$ and a confidence level $\alpha \in (0, 1)$, a composite model $\hat{f}$ is $(\varepsilon, \alpha)$ group-conditional PAC-efficient with respect to partition $\mathcal{G}$ if for $G_j \in \mathcal{G}$,
\begin{equation*}
\mathbb{P}_{\mathcal{D}_j}(R(\hat{f} \mid G_j) \leq \varepsilon) \geq 1 - \alpha,
\end{equation*}
where $\mathcal{D}_j = \{x_i \sim \mathcal{P} : x_i \in G_j\}$ is the data for group $G_j$.
\end{definition}
This definition is stronger than the marginal PAC guarantee, as it controls performance loss within each group $G_j \in \mathcal{G}$, but weaker than the full conditional PAC guarantees in Eq.~\eqref{eq:ccpac}. 
Group-conditional PAC efficiency offers a practical way to control group-specific performance loss. 
It relaxes the conception of fully conditional PAC efficiency by balancing feasibility with efficiency.
Figure~\ref{fig:clustered-pac} illustrates the trade-off between feasibility and the strength of conditional guarantees, from marginal PAC efficiency to full conditional PAC efficiency.


\begin{remark}{(Group partition)}
The group partition $\mathcal{G}$ may be specified a priori or inferred from data, depending on the availability of group information.
Many benchmarks provide group labels, such as subject areas in MATH-500~\citep{lightman2023lets} and difficulty-level annotations in ZebraLogic~\citep{lin2025zebralogic}.
When no grouping information is available, the partition must be learned from data via clustering (Section~\ref{sec:uncertainty-clustering}).
\end{remark}


\textbf{Group-conditional calibration.}
To achieve group-conditional PAC efficiency, we introduce a statistical framework termed \emph{G-PAC reasoning}.
Suppose we are given a calibration dataset $\mathcal{D}_{\mathrm{cal}}=\{(x_i,y_i)\}_{i=1}^n$, where $y_i = f(x_i)$ denotes the output of the thinking model.

To control the group-conditional risk below a tolerance $\varepsilon$ with confidence level $1-\alpha$, we construct an upper confidence bound (UCB) $\hat{L}_{j,u}(\alpha)$ for each group $G_j$ on the calibration set such that
\begin{equation*}
\mathbb{P}\!\left(\hat{L}_{j,u}(\alpha) \ge R(\hat{f}\mid G_j)\right) \ge 1-\alpha,
\end{equation*}
where the probability is taken over the randomness of the calibration data and the importance sampling procedure.

The calibrated threshold $\hat{u}_j$ for group $G_j$ is then defined as the largest uncertainty level for which the UCB remains below the tolerance $\varepsilon$:
\begin{equation}
\label{eq:u^hat_j}
\hat{u}_{j} = \max \left\{ u \in [0,1] : \hat{L}_{j,u}(\alpha) \le \varepsilon \right\}.
\end{equation}

Following PAC reasoning~\citep{zeng2025pac}, we construct an estimator of the group-conditional performance loss via an importance sampling procedure.
Specifically, we sample indices from the calibration set uniformly with replacement and perform Bernoulli trials to decide whether to query expert (thinking-model) outputs.
This procedure yields i.i.d.\ random variables whose expectation equals the target group-conditional performance loss.
Based on these samples, we can construct a UCB using different statistical tools, including the central limit theorem (CLT), Hoeffding's inequality~\citep{hoeffding1963probability}, and Bernstein's inequality~\citep{bentkus2004hoeffdings}.
In this work, we adopt a CLT-based approach to form a confidence interval for the sample mean, which is suitable for sufficiently large sample sizes.
Algorithm~\ref{alg:ucb} summarizes the resulting CLT-based UCB construction.
Detailed derivations and alternative concentration-based bounds are provided in Appendix~\ref{app:ucb-construction}.


\begin{algorithm}[ht!]
\caption{$\texttt{UCB}(\mathcal{D})$: UCB construction via CLT}
\label{alg:ucb}
\begin{algorithmic}[1]
\REQUIRE Data $\mathcal{D} = \{x_i\}_{i=1}^n$ with model outputs $y_i = f(x_i)$ and $\tilde{y}_i = \tilde{f}(x_i)$, uncertainty scores $\{U_i\}_{i=1}^n$, sampling weights $\{\pi_i\}_{i=1}^n$, sampling size $m$, confidence level $\alpha$
\ENSURE UCB function $\hat{L}_u(\alpha)$
\STATE Initialize $\mathcal{Z} \leftarrow []$
\FOR{$t = 1$ to $m$}
    \STATE Sample index $i_t \sim \text{Unif}(\{1, \ldots, n\})$
    \STATE Sample $\xi_{i_t} \sim \text{Bern}(\pi_{i_t})$
    \STATE $Z_t \leftarrow \xi_{i_t} \cdot \ell(y_{i_t}, \tilde{y}_{i_t}) / \pi_{i_t}$
    \STATE Append $Z_t$ to $\mathcal{Z}$
\ENDFOR
\FOR{candidate $u \in \mathcal{U}$}
    \STATE Compute $Z_t(u) \leftarrow Z_t \cdot \mathbf{1}\{U_{i_t} \leq u\}$ for all $t \in [m]$
    \STATE $\hat{\mu}_Z(u) \leftarrow \frac{1}{m} \sum_{t=1}^m Z_t(u)$
    \STATE $\hat{\sigma}_Z(u) \leftarrow \sqrt{\frac{1}{m-1} \sum_{t=1}^m (Z_t(u) - \hat{\mu}_Z(u))^2}$
    \STATE $\hat{L}_u(\alpha) \leftarrow \hat{\mu}_Z(u) + z_{1-\alpha} \cdot \hat{\sigma}_Z(u) / \sqrt{m}$
\ENDFOR
\STATE {\bfseries Return} $\hat{L}_u(\alpha)$
\end{algorithmic}
\end{algorithm}


The G-PAC reasoning algorithm operates in two phases: calibration and deployment.
In the calibration phase, for each group $G_j$, we use its calibration subset $\mathcal{D}_j$ to estimate group-conditional performance loss and learn a group-specific threshold $\hat{u}_j$ by constructing an upper confidence bound $\hat{L}_{j,u}(\alpha)$ such that $\hat{L}_{j,\hat{u}_j}(\alpha) \le \varepsilon$.
In the deployment phase, each test input $x$ is assigned to its corresponding group $G_j$ and routed according to the learned threshold: it is processed by the non-thinking model if $U(x) \le \hat{u}_j$, and otherwise by the thinking model.
Algorithm~\ref{alg:calibration} summarizes the complete procedure of calibration, while the theoretical analysis is provided in Section~\ref{sec:theoretical-analysis}.
Algorithm~\ref{alg:calibration} assumes that the group partition $\mathcal{G}$ is known \emph{a priori}. When this assumption does not hold, we introduce Clustered PAC (C-PAC) reasoning to infer the partition from calibration data.

\begin{algorithm}[ht!]
\caption{G-PAC reasoning}
\label{alg:calibration}
\begin{algorithmic}[1]
\REQUIRE Calibration set $\mathcal{D}_{\text{cal}}$, group partition $\mathcal{G} = \{G_1, \ldots, G_k\}$, error tolerance $\varepsilon$, confidence $\alpha$
\ENSURE Group-specific thresholds $\{\hat{u}_1, \ldots, \hat{u}_k\}$ and composite model $\hat{f}$
\FOR{each group $G_j \in \mathcal{G}$}
    \STATE $\mathcal{D}_j \leftarrow \{x_i \in \mathcal{D}_{\text{cal}} : x_i \in G_j\}$
    \STATE $\hat{L}_{j,u}(\alpha) \leftarrow \texttt{UCB}(\mathcal{D}_j)$ 
    \STATE $\hat{u}_j \leftarrow \max\{u : \hat{L}_{j,u}(\alpha) \leq \varepsilon\}$
    \STATE $\hat{f}_j(x) \leftarrow \begin{cases}
        \tilde{f}(x) & \text{if } U(x) \leq \hat{u}_j \\
        f(x) & \text{otherwise}
    \end{cases}$
\ENDFOR
\STATE $\hat{f}(x) \leftarrow \sum_{j=1}^{k} \hat{f}_j(x) \cdot \mathbf{1}\{x \in G_j\}$
\STATE {\bfseries Return} $\{\hat{u}_1, \ldots, \hat{u}_k\}$, and $\hat{f}$
\end{algorithmic}
\end{algorithm}

The theoretical analysis for G-PAC reasoning is presented in Section~\ref{sec:known-analysis}. Specifically, Theorem~\ref{thm:known-grouping} proves that each group achieves the PAC guarantee, and Theorem~\ref{thm:empirical-known-grouping} provides a finite-sample PAC guarantee.

\subsection{Clustered PAC (C-PAC) reasoning}
\label{sec:uncertainty-clustering}

When the group partition is not known in advance, we learn it from calibration data.
We use uncertainty-based clustering to discover groups from the uncertainty score $U(x)$.

\paragraph{Clustering in $U$ space.}
We map each calibration point $x_i$ to a scalar score $U_i = U(x_i)$, and cluster the values $\{U_i\}$.
This is a one-dimensional clustering problem, which is simpler than clustering in the high-dimensional input space $\mathcal{X}$ (or its embedding space).
In practice, we can use standard methods such as 1D k-means or hierarchical clustering to obtain groups with similar uncertainty levels.

We consider two ways to combine clustering and calibration. In the split approach, we split the calibration set into two disjoint subsets, $\mathcal{D}_{\text{cluster}}$ and $\mathcal{D}_{\text{cal}}$: we learn the partition from $\mathcal{D}_{\text{cluster}}$ and calibrate thresholds on $\mathcal{D}_{\text{cal}}$, so the learned partition is independent of the threshold calibration. This yields exact coverage per learned group. In the joint approach, we use the same set for both steps, $\mathcal{D}_{\text{cluster}}=\mathcal{D}_{\text{cal}}$, which is more sample-efficient but makes the learned partition and the thresholds dependent; we account for this by adding an extra gap term $\delta(n,k)$ in the UCB construction. The theoretical analysis is given in Section~\ref{sec:unknown-analysis}, and Theorem~\ref{thm:unknown-grouping} provides PAC guarantees for both approaches: the split approach achieves exact coverage for each learned group, while the joint approach has an extra gap term $c \cdot \delta(n, k)$.

\section{Theoretical analysis}
\label{sec:theoretical-analysis}

This section establishes the theoretical guarantees for G-PAC reasoning.
We prove group-wise validity for known partitions (Section~\ref{sec:known-analysis}), show that oracle grouping improves efficiency under heterogeneity (Section~\ref{sec:oracle-efficiency}), and analyze the theoretical results of learned partitions when group partitions are not available \emph{a priori} (Section~\ref{sec:unknown-analysis}).


\subsection{Validity of G-PAC reasoning}
\label{sec:known-analysis}


\begin{assumption}[UCB validity]
\label{assump:ucb-validity}
For each group $G_j$, threshold $u$, and a confidence level $\alpha \in (0, 1)$, the upper confidence bound $\hat{L}_{j,u}(\alpha)$ computed on the group set $\mathcal{D}_j$ satisfies
\begin{equation*}
\mathbb{P}_{\mathcal{D}_j}(R(\hat{f} \mid G_j) \leq \hat{L}_{j,u}(\alpha)) \geq 1 - \alpha,
\end{equation*}
where $R(\hat{f} \mid G_j)$ is the group-conditional performance loss of the composite model $\hat{f}_u$ with threshold $u$, and the probability is taken over the randomness of the group calibration set $\mathcal{D}_j = \{x_i \in \mathcal{D}_{\text{cal}} : x_i \in G_j\}$.
\end{assumption}

In this work, we use bootstrap-based methods combined with the central limit theorem to satisfy the assumptions above, as described in Algorithm~\ref{alg:ucb}.
Alternatively, the same assumption can be instantiated using standard concentration inequalities, such as Hoeffding's inequality~\citep{hoeffding1963probability} or empirical Bernstein bounds~\citep{howard2021timeuniform}.

\begin{theorem}[PAC guarantee for known grouping]
\label{thm:known-grouping}
Let $\mathcal{G} = \{G_1, \ldots, G_k\}$ be a known partition of the input space.
Suppose the loss function $\ell: \mathcal{Y} \times \mathcal{Y} \to [0, B]$ is bounded for some constant $B > 0$, and Assumption~\ref{assump:ucb-validity} holds.
If Algorithm~\ref{alg:calibration} is applied with calibration set $\mathcal{D}_{\text{cal}}$, error tolerance $\varepsilon > 0$, and confidence level $\alpha \in (0,1)$, for each group $G_j$, we have: 
\begin{equation}
\label{eq:known-guarantee}
\mathbb{P}_{\mathcal{D}_j}(R(\hat{f} \mid G_j) \leq \varepsilon) \geq 1 - \alpha.
\end{equation}
\end{theorem}
Eq.~\eqref{eq:known-guarantee} states the group-wise guarantee.
See Appendix~\ref{app:proof-known} for the proof. We also provide an empirical version of the guarantee that bounds the test risk on a finite test set.

\begin{theorem}[PAC finite-sample guarantee for known grouping]
\label{thm:empirical-known-grouping}
Suppose the loss function $\ell: \mathcal{Y} \times \mathcal{Y} \to [0, B]$ is bounded for some constant $B > 0$, Assumption~\ref{assump:ucb-validity} holds, and the test set $\mathcal{D}_{\text{test}}$ is independent of the calibration set $\mathcal{D}_{\text{cal}}$.
Let $\mathcal{G} = \{G_1, \ldots, G_k\}$ be a known partition, and let $\hat{u}_j$ be the threshold selected by Algorithm~\ref{alg:calibration} for group $G_j$.
Given $\varepsilon, \alpha \in (0,1)$, for each group $G_j$ with $N_j$ test samples, and any $t > 0$,
\begin{equation*}
\label{eq:empirical-known-guarantee}
\mathbb{P}_{\mathcal{D}_j}\left(\hat{R}(\hat{f} \mid G_j) \leq \varepsilon + t\right) \geq 1 - \alpha - \exp\left(-\frac{2 N_j t^2}{B^2}\right),
\end{equation*}
where $\hat{R}(\hat{f} \mid G_j) = \frac{1}{N_j} \sum_{i: x_i \in G_j} \ell(\hat{f}(x_i), f(x_i))$ is the empirical risk on the test set for group $G_j$.
\end{theorem}
See Appendix~\ref{app:proof-empirical-known} for the proof.

\begin{remark}
A common case is bounded loss $\ell \in [0, 1]$, e.g., 0-1 loss for verifiable answers. Then, the bound simplifies to $\mathbb{P}(\hat{R}(\hat{f} \mid G_j) \leq \varepsilon + t) \geq 1 - \alpha - e^{-2 N_j t^2}$.
This provides exact risk control for each group with probability at least $1 - \alpha$ up to a slack $t$ that decreases with test set size.
\end{remark}

\subsection{Oracle partition}
\label{sec:oracle-efficiency}

We characterize the oracle optimal partition and prove that grouping strictly improves efficiency under heterogeneity.
We formalize the notion of an oracle group partition in the PAC framework.
For an error tolerance $\varepsilon > 0$, and a confidence parameter $\alpha \in (0, 1)$, we say a threshold $u_j$ is $(\varepsilon, \alpha)$-\textit{PAC feasible} for group $G_j$ if:
\begin{equation*}
\mathbb{P}(R(u_j \mid G_j) \leq \varepsilon) \geq 1 - \alpha,
\end{equation*}
where $R(u \mid G_j) = \mathbb{E}_{x \sim \mathcal{P}}[\ell(\tilde{f}(x), f(x)) \cdot \mathbf{1}\{U(x) \leq u\} \mid x \in G_j]$ is the performance loss conditional on $G_j$, and the probability is over the randomness of the calibration data $\mathcal{D}_j$ for group $G_j$. The PAC-optimal threshold for group $G_j$ is the largest feasible threshold:
\begin{equation}
\label{eq:pac-optimal-threshold}
\hat{u}_j^*(\varepsilon, \alpha) = \sup\{u : u \text{ is } (\varepsilon, \alpha)\text{-PAC feasible for } G_j\}.
\end{equation}
Then we define the oracle optimal partition that maximizes the use of the non-thinking model while maintaining PAC guarantees: 
\begin{definition}[Oracle optimal partition] 
\label{def:oracle-optimal}
Given an error tolerance $\varepsilon > 0$, a confidence parameter $\alpha \in (0, 1)$, and number of groups $k$, the oracle optimal partition $\mathcal{G}^* = \{G_1^*, \ldots, G_k^*\}$ is defined as:
\begin{equation*}
\mathcal{G}^* = \arg\max_{\mathcal{G} \in \Pi_k(\mathcal{X})} \sum_{j=1}^k p_j \cdot \mathbb{P}(U(x) \leq \hat{u}_j^*(\varepsilon, \alpha) \mid x \in G_j),
\end{equation*}
where $\Pi_k(\mathcal{X})$ denotes the set of all measurable $k$-partitions of the input space $\mathcal{X}$, $p_j = \mathbb{P}(x \in G_j)$ is the probability mass of group $G_j$, and $\hat{u}_j^*(\varepsilon, \alpha)$ is the PAC-optimal threshold for group $G_j$ defined in~Eq.~\eqref{eq:pac-optimal-threshold}.
\end{definition}

\begin{remark}[Efficiency as non-thinking model usage]
  \label{remark:efficiency-interpretation}
  The PAC efficiency $\text{Eff}_{\text{PAC}}(\mathcal{G}; \varepsilon, \alpha)$ equals the probability of routing an input to the non-thinking model under the PAC-optimal thresholds.
  Maximizing efficiency is equivalent to maximizing the use of the non-thinking model while satisfying the PAC guarantee.
  Since the non-thinking model is faster and cheaper than the thinking model, higher efficiency directly translates to lower computational cost.
\end{remark}

\textbf{Benefits of grouping.}
Intuitively, the oracle optimal partition maximizes the usage of the non-thinking model while ensuring that the performance loss of each group satisfies the $(\varepsilon, \alpha)$-PAC guarantee.
A natural question is whether grouping provides any benefit over the marginal approach that uses a single threshold for all inputs.
The following proposition shows that grouping always weakly improves efficiency, with strict improvement when groups exhibit heterogeneous risk profiles:
\begin{proposition}[Group-conditional benefit]
\label{prop:grouping-benefit}
For the optimal partition $\mathcal{G}$, the PAC efficiency satisfies
\begin{equation*}
\text{Eff}_{\text{PAC}}(\mathcal{G}; \varepsilon, \alpha) \geq \text{Eff}_{\text{PAC}}(\{\mathcal{X}\}; \varepsilon, \alpha),
\end{equation*}
where $\{\mathcal{X}\}$ denotes the trivial input space (no grouping).
Equality holds if and only if the PAC-optimal thresholds are identical across all groups:$\hat{u}_1^*(\varepsilon, \alpha) = \cdots = \hat{u}_k^*(\varepsilon, \alpha) = \hat{u}^*(\varepsilon, \alpha),
$
where $\hat{u}^*(\varepsilon, \alpha)$ is the PAC-optimal threshold.
\end{proposition}
See Appendix~\ref{app:proof-grouping-benefit} for the proof.
Heterogeneity in the conditional risk functions $R(u \mid G_j)$ arises from different data distributions or non-uniform model capability across groups. 


\subsection{Validity of C-PAC reasoning}
\label{sec:unknown-analysis}

Let $\hat{\mathcal{G}} = \{\hat{G}_1, \ldots, \hat{G}_k\}$ be a learned partition and let $\mathcal{G}^* = \{G_1^*, \ldots, G_k^*\}$ be the oracle one.
Let $\hat{g}, g^*: \mathcal{X} \to [k]$ be their group assignment functions.
We define the \emph{partition gap} as
\begin{equation}
\label{eq:partition-gap}
\delta := \min_{\sigma \in S_k} \mathbb{P}_{x \sim \mathcal{P}}\left(\hat{g}(x) \neq \sigma(g^*(x))\right),
\end{equation}
where $S_k$ is the set of permutations on $[k]$.
When $\delta = 0$, the learned partition matches the oracle exactly (up to relabeling).
We establish coverage guarantees under two approaches: sample splitting uses disjoint sets $\mathcal{D}_{\text{cluster}}$ and $\mathcal{D}_{\text{cal}}$ for clustering and threshold calibration, while the joint approach uses the same set $\mathcal{D}_{\text{cal}}$ for both tasks.

\begin{theorem}[PAC guarantee for unknown grouping]
\label{thm:unknown-grouping}
Let $\hat{\mathcal{G}} = \{\hat{G}_1, \ldots, \hat{G}_k\}$ be a learned partition and let $\delta$ be the partition gap in Eq.~\eqref{eq:partition-gap}.
Suppose the loss function \(\ell\) is bounded for some constant $B > 0$, and Assumption~\ref{assump:ucb-validity} holds.
If Algorithm~\ref{alg:calibration} is applied with error tolerance $\varepsilon > 0$ and confidence parameter $\alpha \in (0,1)$, then:
\begin{enumerate}
\item \textbf{Sample splitting:} If $\hat{\mathcal{G}}$ is learned from $\mathcal{D}_{\text{cluster}}$ and calibration uses $\mathcal{D}_{\text{cal}}$ (disjoint from $\mathcal{D}_{\text{cluster}}$), then for each learned group $\hat{G}_j$,
\begin{equation}
\label{eq:unknown-guarantee-split}
\mathbb{P}_{\mathcal{D}_{\text{cal}}}(R(\hat{f} \mid \hat{G}_j) \leq \varepsilon) \geq 1 - \alpha.
\end{equation}
And, the efficiency satisfies $\text{Eff}_{\text{PAC}}(\mathcal{G}^*; \varepsilon, \alpha) - \text{Eff}_{\text{PAC}}(\hat{\mathcal{G}}; \varepsilon, \alpha) \leq 2\delta$.

\item \textbf{Joint approach:} If $\hat{\mathcal{G}}$ is learned from the calibration set $\mathcal{D}_{\text{cal}}$ of size $n$ (same set used for both clustering and calibration), then with probability at least $1 - \alpha$,
\begin{equation}
\label{eq:unknown-guarantee-joint}
R(\hat{f} \mid \hat{G}_j) \leq \varepsilon + c \cdot \delta \quad \text{for all } j = 1, \ldots, k,
\end{equation}
for some positive constant \(c\).
\end{enumerate}
\end{theorem}

Eq.~\eqref{eq:unknown-guarantee-split} and \eqref{eq:unknown-guarantee-joint} analyze the two settings:
Sample splitting uses disjoint data for partition learning and calibration, giving exact coverage regardless of the partition gap $\delta$.
Joint learning reuses the same data, improving efficiency but incurring a coverage gap of order $c\delta$.
In both cases, a larger $\delta$ reduces efficiency, but under standard conditions $\delta=O(n^{-\beta})$~\citep{lu2016statistical, luxburg2008consistency}, so the joint gap vanishes as $n$ grows.

\section{Experiments}
\label{sec:experiments}
In this section, we evaluate G-PAC (for known partitions) and C-PAC (for learned partitions) across multiple LLM benchmarks and deployment settings.
Specifically, we investigate two key questions:
(i) \textbf{Group-conditional risk control}: Can G-PAC and C-PAC effectively control the performance loss for each group, especially when vanilla PAC reasoning fails?
(ii) \textbf{Trade-off and cost}: What is the cost of achieving group-level risk control in terms of efficiency? 

\begin{figure*}[t]
    \centering
    \begin{subfigure}[t]{\linewidth}
        \centering
        \includegraphics[width=0.27\linewidth]{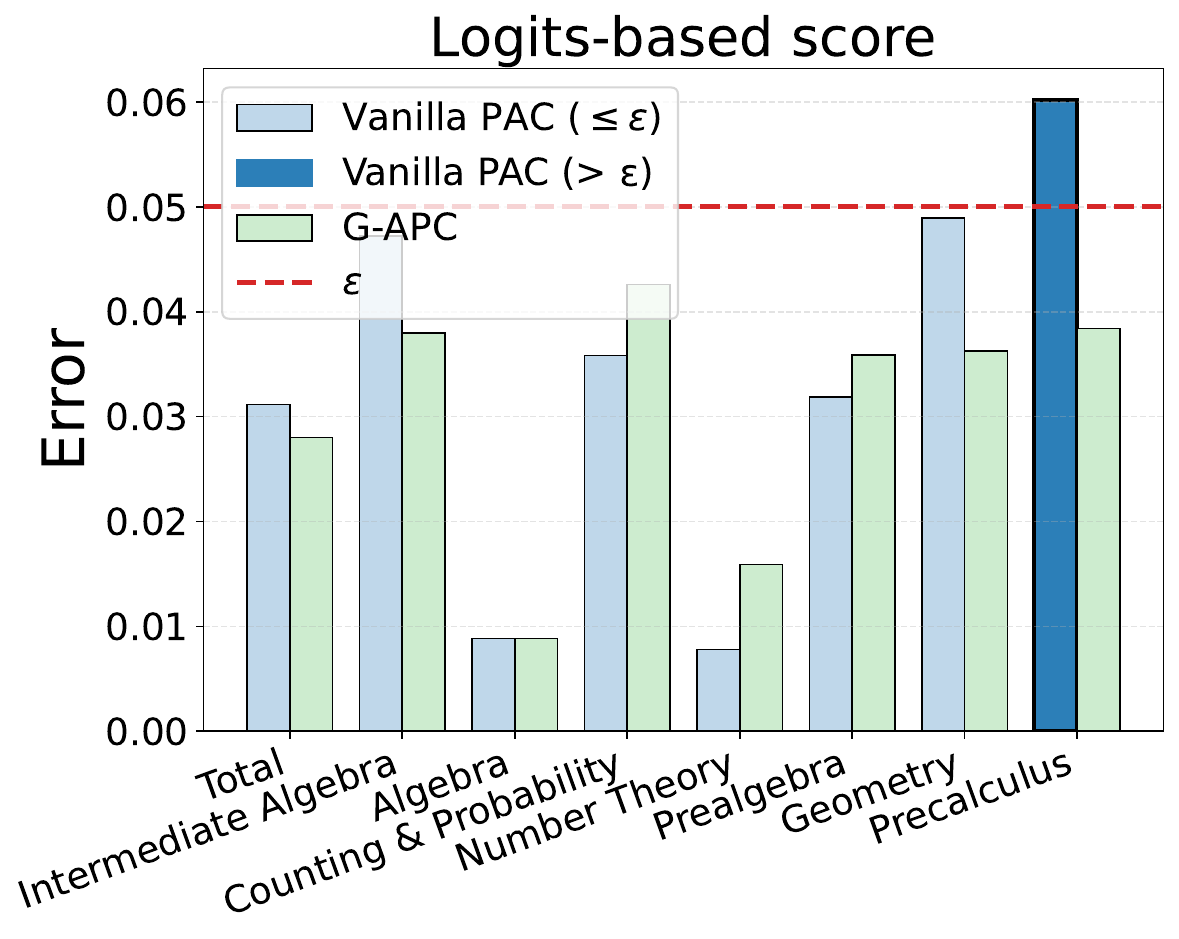}
        \includegraphics[width=0.27\linewidth]{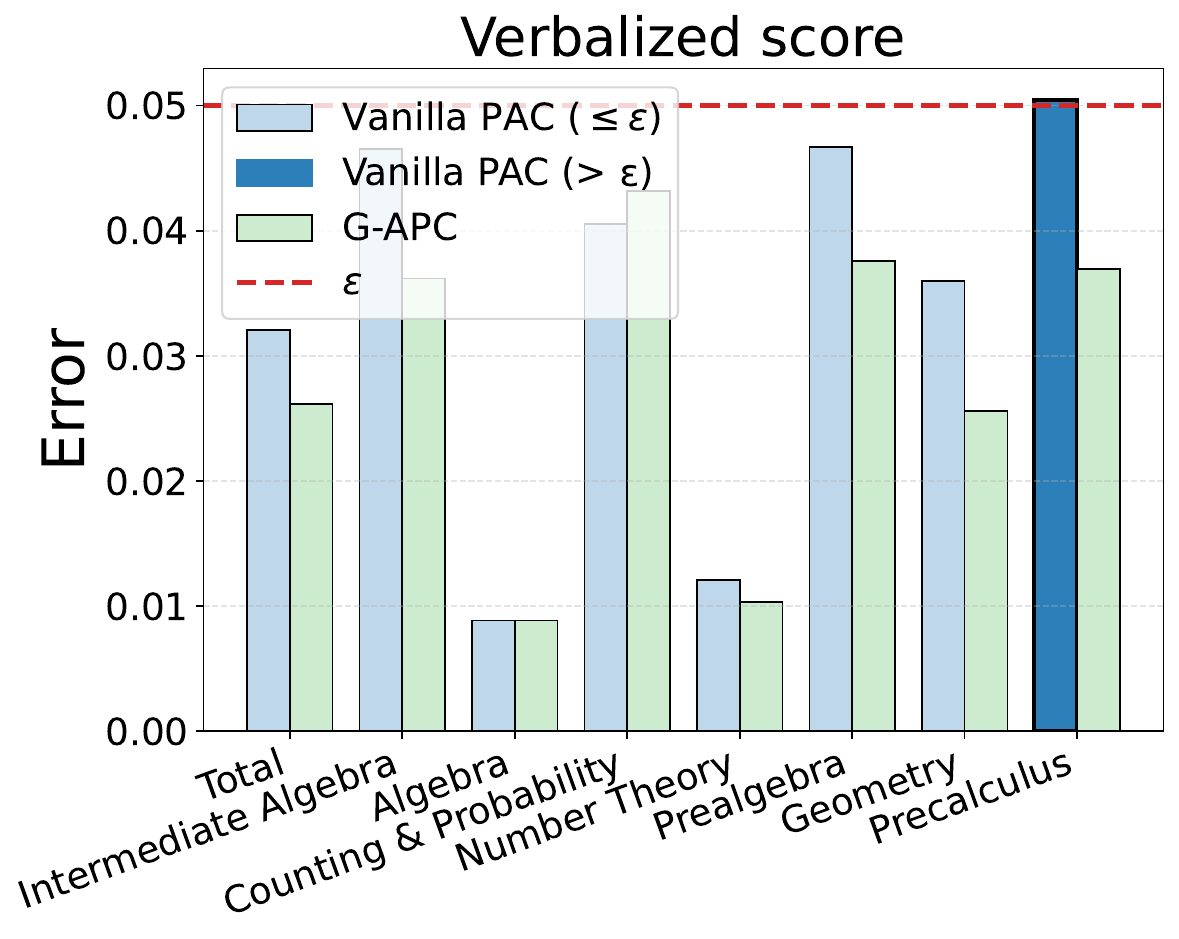}
        \includegraphics[width=0.27\linewidth]{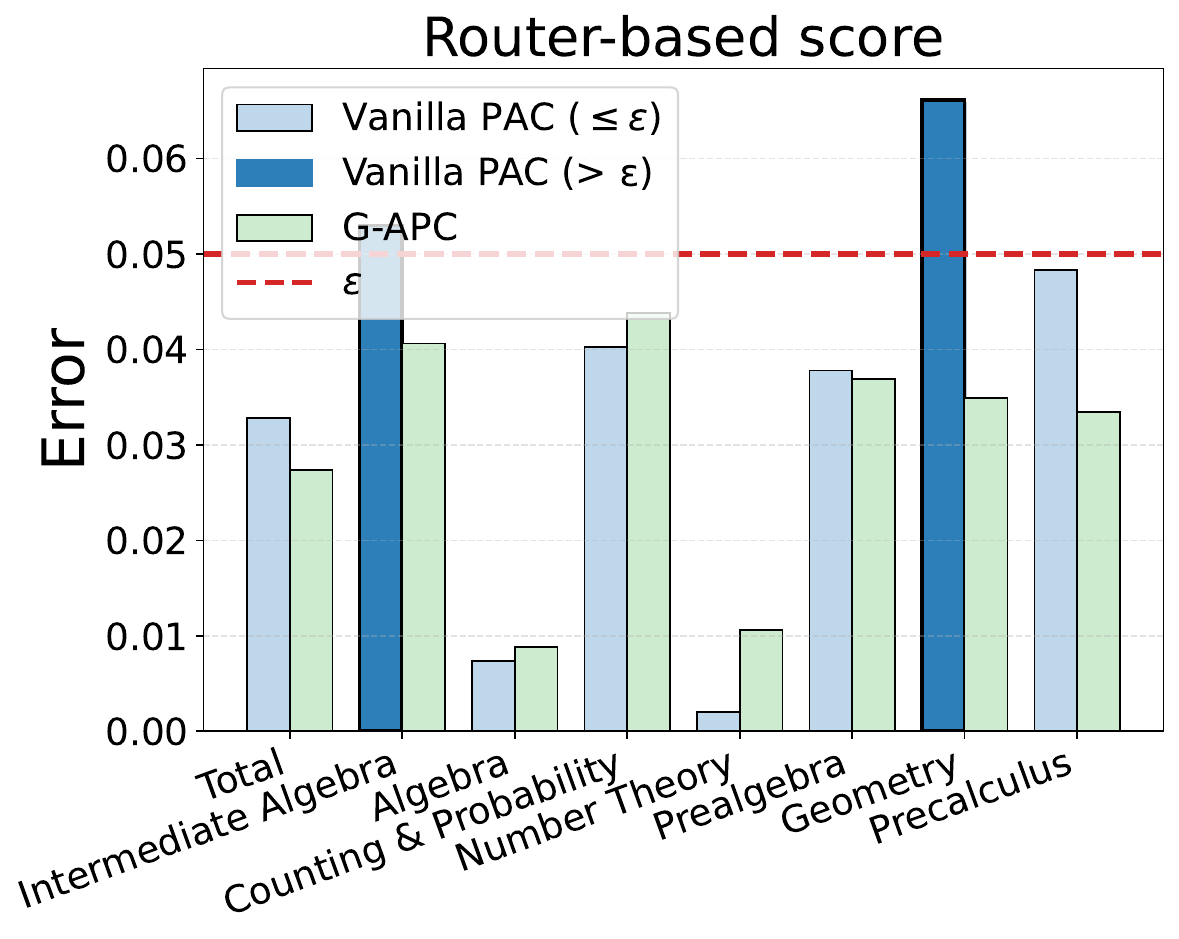}
        \caption*{MATH-500}
    \end{subfigure}\\
    \begin{subfigure}[t]{\linewidth}
        \centering
        \includegraphics[width=0.27\linewidth]{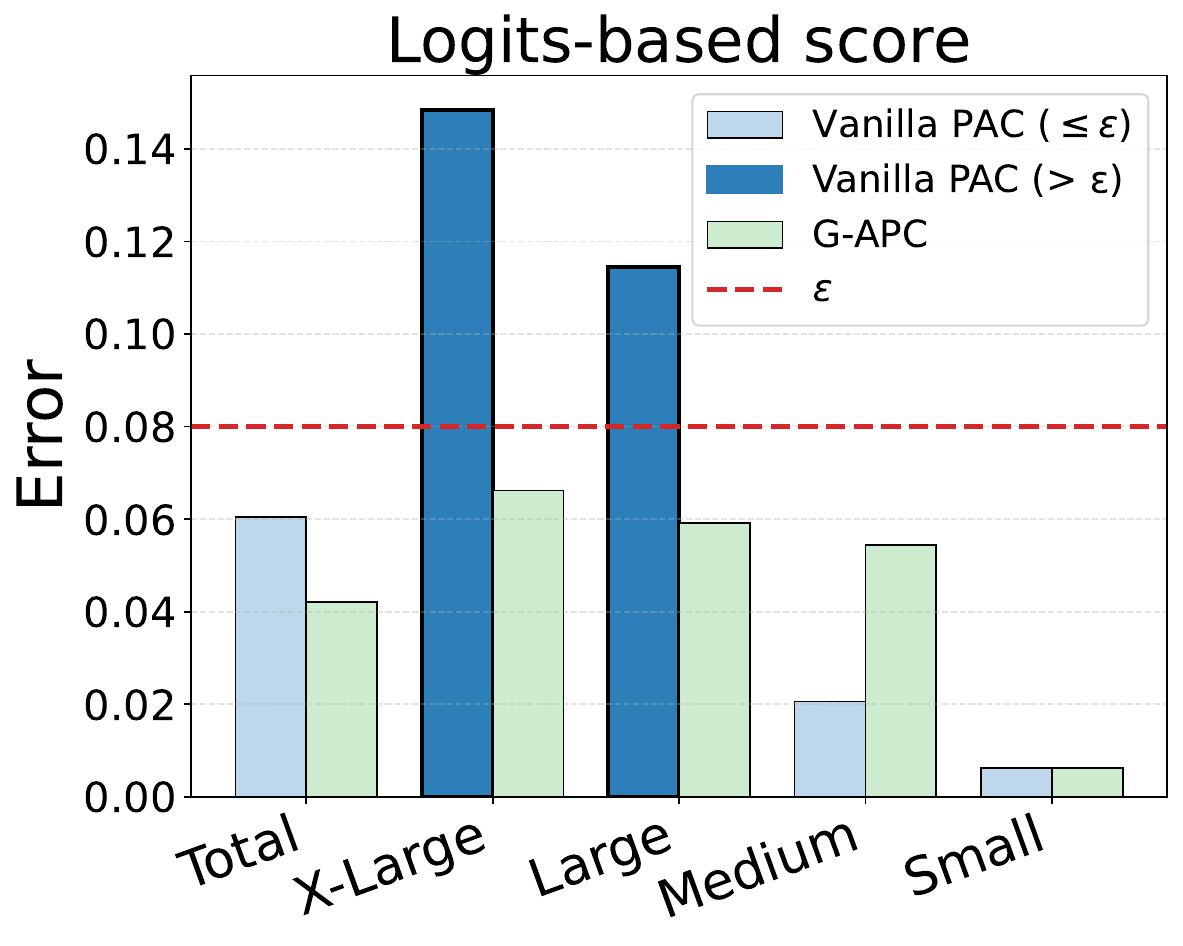}
        \includegraphics[width=0.27\linewidth]{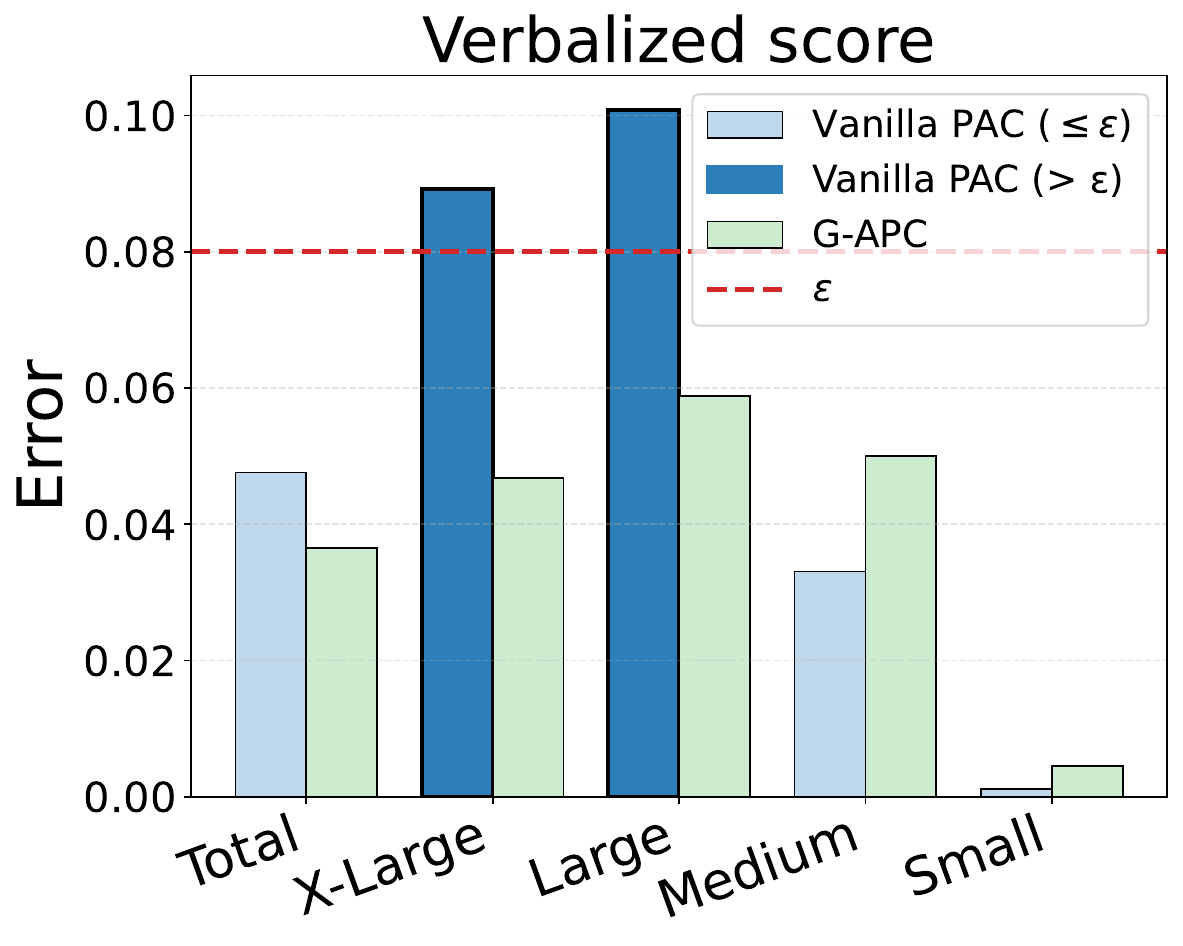}
        \includegraphics[width=0.27\linewidth]{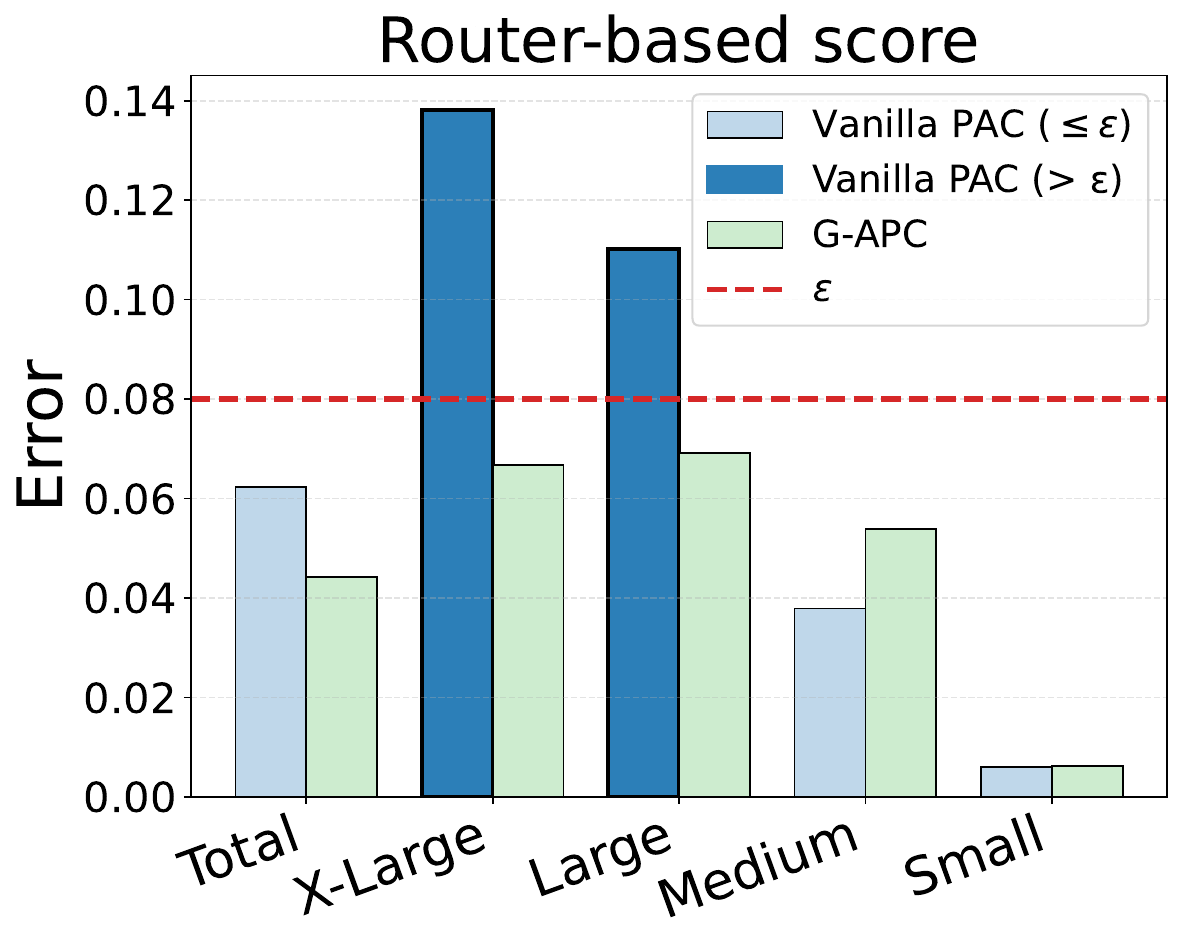}
        \caption*{ZebraLogic}
    \end{subfigure}
    \begin{subfigure}[t]{\linewidth}
        \centering
        \includegraphics[width=0.27\linewidth]{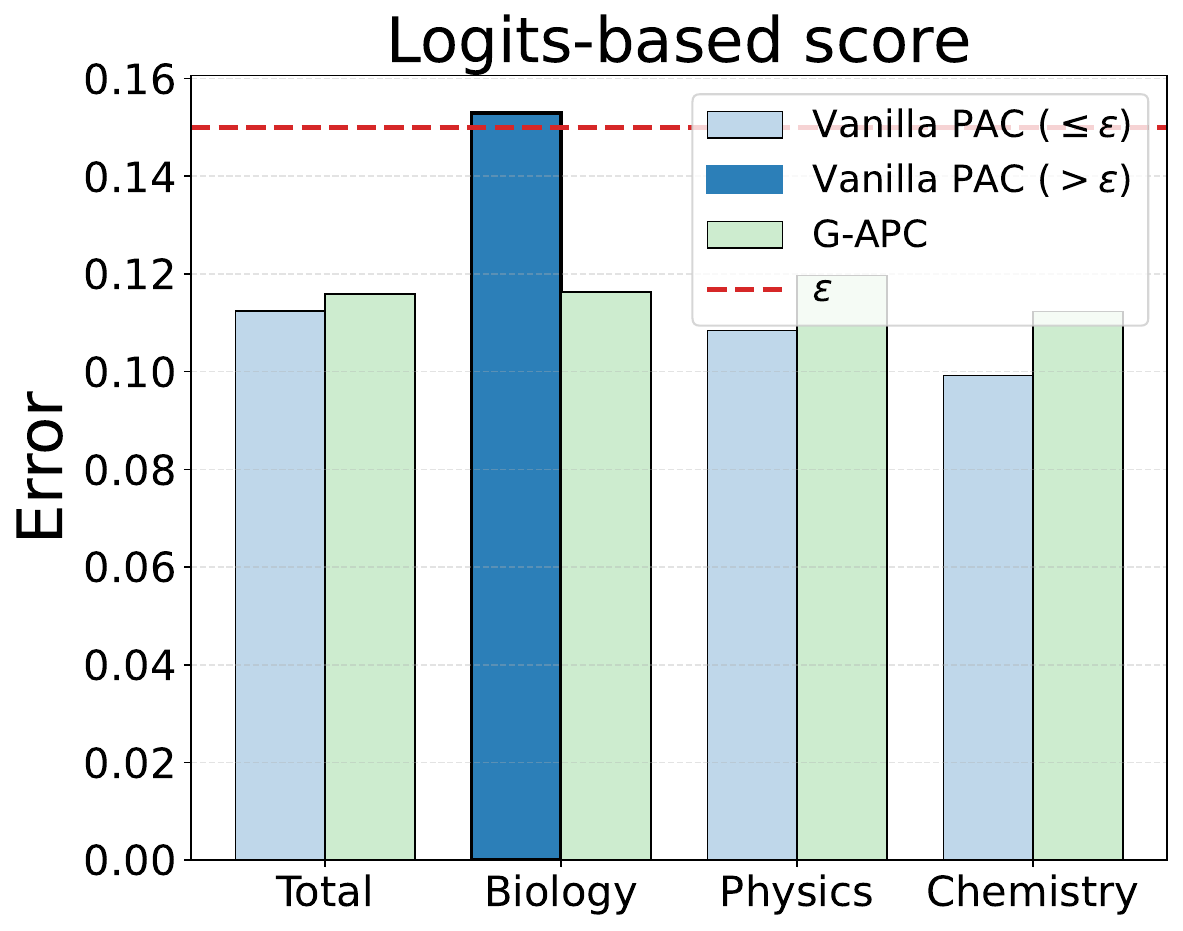}
        \includegraphics[width=0.27\linewidth]{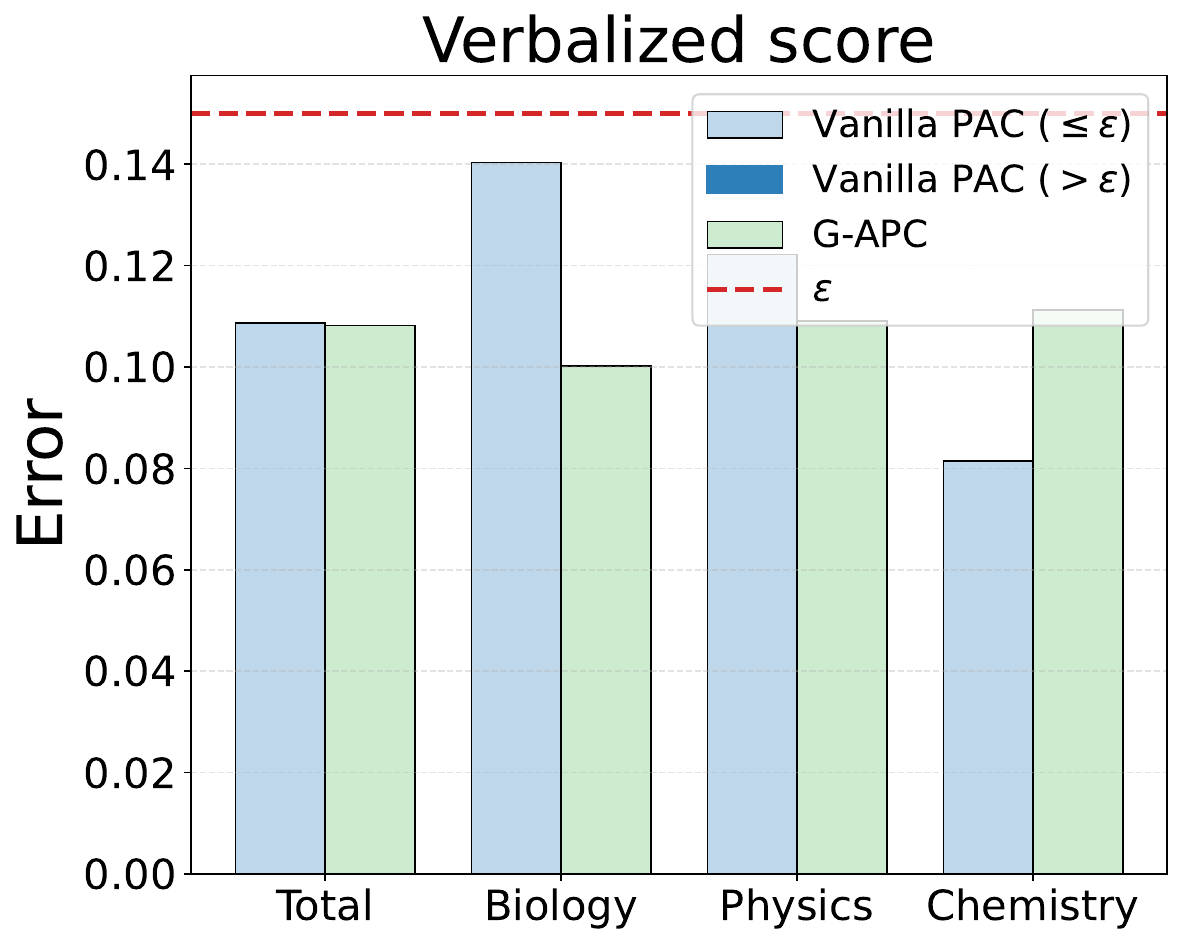}
        \includegraphics[width=0.27\linewidth]{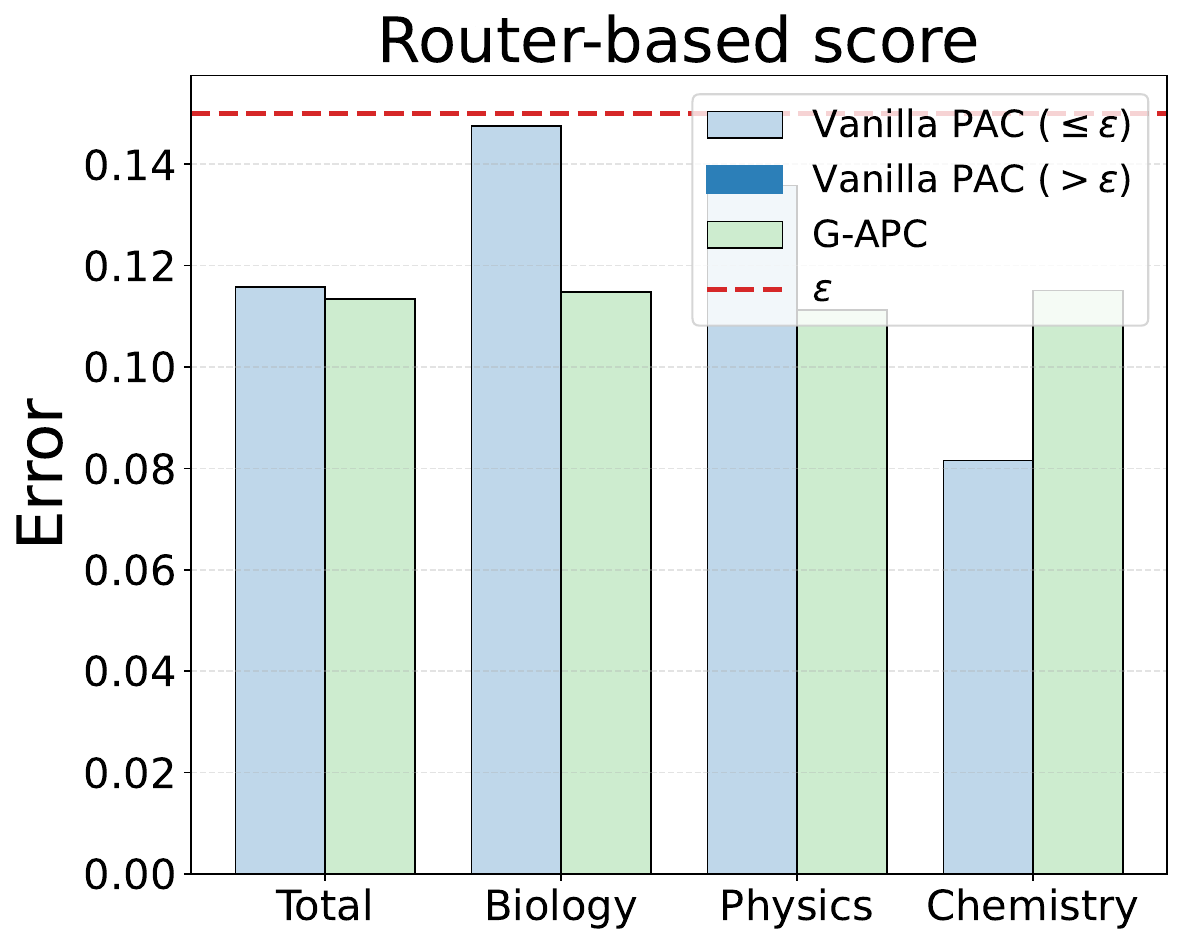}
        \caption*{GPQA}
    \end{subfigure}
    \caption{
        \textbf{G-PAC controls the group-conditional performance loss below the target while vanilla PAC fails (dark blue) across three different uncertainty scores.} All results are obtained using the Qwen model.
   The figure reports the overall and per-category performance losses on three reasoning benchmarks.}
    \label{fig:errorbar_category}
\end{figure*}

\subsection{Setup}
\label{sec:exp_setup}

\textbf{Large language models and Datasets.} We evaluate the PAC reasoning based on Qwen3 series models~\citep{yang2025qwen3} and Llama-3.1-8B–based models. 
Specifically, we employ the ``Qwen3-4B-Thinking-2507'' as the thinking LLM and ``Qwen3-4B-Instruct-2507'' as the non-thinking LLM. Moreover, we use ``DeepSeek-R1-Distill-Llama-8B''~\citep{deepseek-ai2025deepseekr1} as the thinking model and ``Llama-3.1-8B-Instruct''~\citep{grattafiori2024llama} as the non-thinking model. Further details are provided in Appendix~\ref{app:llm_datasets}.
We conduct experiments on a range of real-world benchmarks, including MATH-500~\citep{lightman2023lets}, ZebraLogic~\citep{lin2025zebralogic}, and GPQA~\citep{rein2024gpqa}.
Additional details can be found in Appendix~\ref{app:llm_datasets}.

\textbf{Uncertainty scores.} In our experiments, we consider three uncertainty scores to support different deployment settings: a logits-based score, a verbalized score, and a router-based score, with definitions provided in Appendix~\ref{app:uncertainty_score}.

\textbf{Baselines and evaluation metrics.} We compare the proposed G-PAC reasoning with the naive PAC reasoning~\citep{zeng2025pac}.
To evaluate risk control under both marginal and group-conditional settings, we adopt two evaluation metrics, i.e., \text{Error} and $\text{Error}_{\text{Gap}}$, to measure the empirical risk and its gap on the test dataset.
In addition, we use \textit{Saved Token Percentage} (STP) to quantify computational cost savings.
Details are described in detail in Appendix~\ref{app:evalution_metrics}.

Across all experiments, we fix the confidence level at $\alpha = 0.05$ and the sampling weight at $\pi = 0.5$, while varying $\varepsilon$.
Each experiment is repeated 100 times, and we report the mean error and cost savings, using the binary loss to quantify performance loss (see Appendix~\ref{app:loss_function}).


\begin{figure*}[!t]
    \centering
    \begin{subfigure}[t]{\linewidth}
        \centering
        \includegraphics[width=0.31\linewidth]{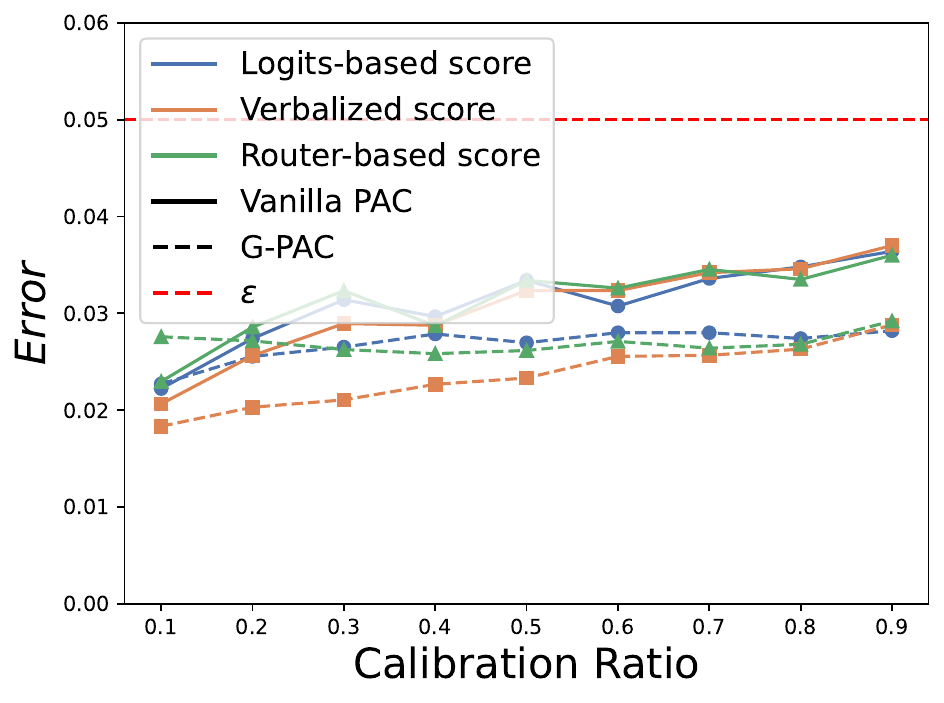}
        \includegraphics[width=0.31\linewidth]{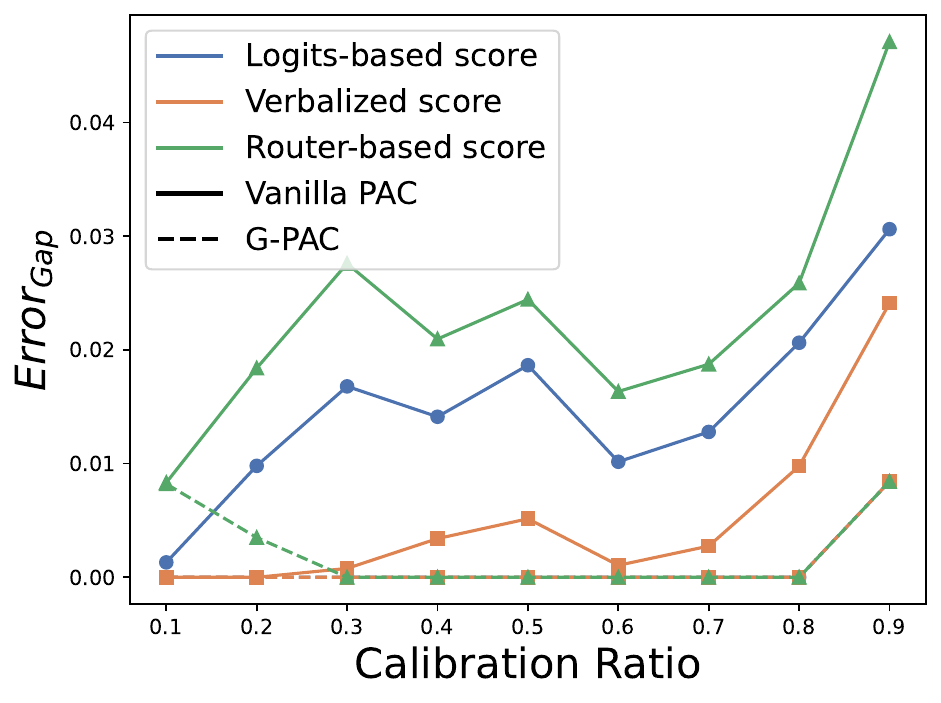}
        \includegraphics[width=0.31\linewidth]{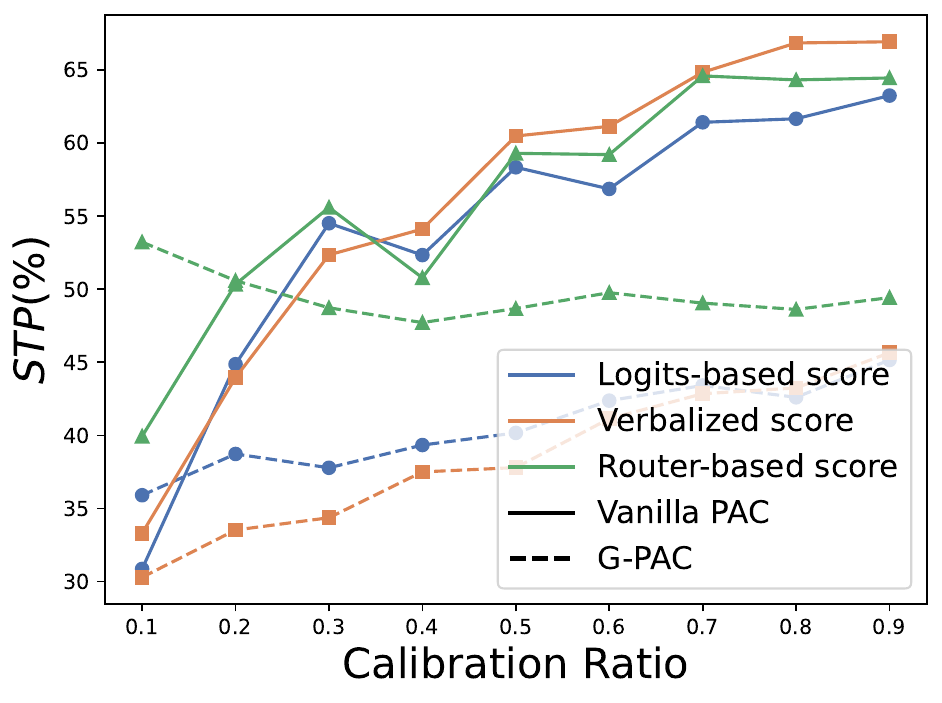}
    \end{subfigure}
    \caption{\textbf{More calibration samples enhance efficiency of reasoning.}
    Experimental results of G-PAC reasoning for different calibration ratios on MATH-500.  The red dashed line $\varepsilon$ means the target risk level.  }
    \label{fig:results_calibration_ratio_math500}
\end{figure*}

\subsection{Results}
\textbf{G-PAC and C-PAC successfully control the group-conditional performance loss below the target while vanilla PAC fails.}
Table~\ref{table:01_results_known} and Table~\ref{table:01_results_unknown} report the results of G-PAC reasoning under known and unknown partitions.
In the known case, we use the predefined partitions from each dataset. For C-PAC, groups are made by clustering uncertainty scores into 3 bins.
Across all datasets, vanilla PAC reasoning fails to control the group-wise error, leading to non-zero $\text{Error}_{\text{Gap}}$.
Specifically, as shown in Figure~\ref{fig:errorbar_category}, vanilla PAC reasoning has high error rates in some groups, such as the \textit{precalculus} category in MATH-500, the \textit{X-Large} category in Zebra Logic, and the \textit{biology} category in GPQA.
In contrast, both G-PAC and C-PAC achieve zero group-conditional error gap, showing they can control group-wise risk even when group labels are not known beforehand.
Although G-PAC and C-PAC may slightly reduce STP because of stricter rules, they still maintain good efficiency in most cases and offer a practical way to balance validity and speed.

\textbf{Trade-off between validity and efficiency.} Although G-PAC and C-PAC ensure group-wise error control, partitioning the input space reduces the calibration samples per group, leading to a loss in efficiency (STP). As shown in Table~\ref{table:01_results_known}, G-PAC often has a lower STP than vanilla PAC while maintaining valid risk control. 
The results in Figure~\ref{fig:results_calibration_ratio_math500} further reflect this sample loss cost, as a larger calibration set helps recover some of the efficiency.

\begin{table}[!t]
\centering
\caption{\textbf{G-PAC performs a smaller error gap than vanilla PAC.}
Experimental results of the binary loss function on verifiable datasets ($\alpha=0.05$).  
We set $\varepsilon=0.05$ for MATH-500, $\varepsilon=0.08$ for ZebraLogic, and $\varepsilon=0.15$ for GPQA.}\label{table:01_results_known}
\setlength{\tabcolsep}{1mm}{
\scriptsize
\resizebox{1\columnwidth}{!}{
\begin{tabular}{lccccccc}
\toprule
\multirow{2}{*}{Dataset}  & \multirow{2}{*}{Metric}  & \multicolumn{2}{c}{Logits-based score}& \multicolumn{2}{c}{Verbalized score} & \multicolumn{2}{c}{Router-based score}
\\
\cmidrule(r){3-4} \cmidrule(r){5-6} \cmidrule(r){7-8}
& & PAC & G-PAC & PAC & G-PAC & PAC & G-PAC  \\
\midrule
\multirow{3}{*}{MATH-500}
& $\text{Error}$ (\%)  & $ 3.08 $ & $2.78$ & $3.24$ & $2.57$ & 3.26 &  2.63 \\
& $\text{Error}_{\text{Gap}}$ (\%)  & $  1.02$ & \colorbox{green!10}{0.00} & $0.11$ & \colorbox{green!10}{0.00} & 1.63 & \colorbox{green!10}{0.00} \\
& STP (\%) $\uparrow$ & $56.86  $  & $43.44 $ & $61.14 $ & $41.63 $ & 59.20 &49.30 \\
\midrule
\multirow{3}{*}{ZebraLogic}
& $\text{Error}$ (\%) & 6.04 & 4.21 & 4.76 &  4.01 & 6.23 &  4.35 \\
& $\text{Error}_{\text{Gap}}$ (\%) & 10.30 & \colorbox{green!10}{0.00} & 3.00 & \colorbox{green!10}{0.00} & 8.84  & \colorbox{green!10}{0.00}\\
& STP (\%)  $\uparrow$& 27.65  &  10.90 &  4.97 & 8.53 & 34.54 & 34.32 \\
\midrule
\multirow{3}{*}{GPQA}
& $\text{Error}$ (\%) & 11.24 & 11.59  &  10.86 & 10.82 & 11.57 & 11.34\\
& $\text{Error}_{\text{Gap}}$ (\%) & 0.30  & \colorbox{green!10}{0.00} & 0.00 & \colorbox{green!10}{0.00} & 0.00 & \colorbox{green!10}{0.00} \\
& STP (\%)  $\uparrow$& 8.73  &  13.78 &  17.07& 19.60  &34.21 & 30.54\\
\bottomrule
\end{tabular}}}
\end{table}

\subsection{Additional discussion}

\textbf{Open-domain task.} We further evaluate our method on an open-domain benchmark, Arena-Hard~\citep{li2025crowdsourced}. 
Since answers in open-domain settings cannot be automatically verified and Arena-Hard does not provide pre-defined group labels, 
we adopt the semantic loss in Eq.~\eqref{eq:semantic_loss} to quantify performance loss and apply C-PAC reasoning.
Table~\ref{table:arena_results} shows that C-PAC reasoning (with 3 groups) consistently achieves zero $\text{Error}_{\text{Gap}}$, while the vallina PAC reasoning fails to control group-conditional risk.
Although C-PAC may slightly reduce STP, it still offers significant token savings and remains effective in open-domain settings.

\textbf{Ablation experiment.}
We conduct ablation experiments to investigate the stability of risk control under different calibration set sizes.
Specifically, we vary the calibration ratio and evaluate both PAC and G-PAC reasoning under a binary loss.
The results on MATH-500 are presented in Figure~\ref{fig:results_calibration_ratio_math500}.
Both PAC reasoning and G-PAC reasoning maintain stable marginal error control across different calibration ratios.
Meanwhile, G-PAC reasoning consistently achieves a lower group-conditional error gap than standard PAC reasoning, indicating improved group-wise risk control.
G-PAC preserves stable efficiency gains as the calibration ratio varies.
Results of ZebraLogic and GPQA are shown in Figure~\ref{fig:results_calibration_ratio_other}.


\section{Conclusion}
\label{sec:conclusion}

We proposed G-PAC reasoning, a statistical framework that enables group-conditional risk control for efficient LLM reasoning.
By partitioning the input space into groups, G-PAC strengthens marginal PAC guarantees with finer-grained validity at the group level.
We establish theoretical guarantees for G-PAC, showing that the group-conditional error gap vanishes for both known and learned groupings.
Experiments across reasoning benchmarks demonstrate that G-PAC achieves group-conditional control while maintaining substantial inference-time savings.


\textbf{Limitations}
First, the efficiency gains of PAC reasoning depend on the quality of the uncertainty score: poorly informative uncertainty estimates may reduce the computational savings.
Second, when extending PAC reasoning to unknown groupings, the method requires either sample splitting, which reduces statistical efficiency, or joint clustering, which may introduce coverage gaps.


\section*{Impact Statement}

This paper presents work whose goal is to advance the field of Machine
Learning. There are many potential societal consequences of our work, none
which we feel must be specifically highlighted here.

\bibliography{cpac}
\bibliographystyle{icml2026}

\newpage
\appendix
\onecolumn

\section{Detailed UCB construction}
\label{app:ucb-construction}

We provide detailed derivations for constructing the upper confidence bound (UCB) for the performance loss using importance sampling.

\paragraph{Importance sampling procedure}
Let $\mathcal{I}_{\text{cal},j}$ denote the index set of samples belonging to group $G_j$ in the calibration dataset.
Given a sampling size $m$, we uniformly draw $m$ indices $\{i_1, \dots, i_m\}$ with replacement from $\mathcal{I}_{\text{cal},j}$. 
For each selected index $i_t$, we perform a Bernoulli trial $\xi_{i_t} \sim \text{Bern}(\pi_{i_t})$ to decide whether to query the expert answer $y_{i_t}$, where $\{\pi_{i_t}\}_{t=1}^m$ are sampling weights.
This yields $m$ i.i.d.\ random variables:
\[
Z_t(u) = \ell(y_{i_t}, \tilde{y}_{i_t}) \frac{\xi_{i_t}}{\pi_{i_t}} \mathbf{1}\{U_{i_t} \leq u\}.
\]
The expectation of $Z_t(u)$ equals the target quantity $L(u,G_j)$, since $\mathbb{E}_{\xi_{i_t}}[\xi_{i_t} / \pi_{i_t} \mid i_t] = 1$.
Then, we can estimate an upper bound for $L(u,G_j)$ by computing a confidence interval for the mean of $\{Z_t(u)\}_{t=1}^m$.

\paragraph{CLT-based UCB}
For large $m$, the sample mean $\hat{\mu}_Z(u) = \frac{1}{m} \sum_{j=1}^m Z_j(u)$ is approximately normal by the central limit theorem, and the UCB is given by:
\[
\hat{L}_{j,u}(\alpha) = \hat{\mu}_Z(u) + z_{1-\alpha} \frac{\hat{\sigma}_Z(u)}{\sqrt{m}},
\]
where $\hat{\sigma}_Z(u)$ is the sample standard deviation and $z_{1-\alpha}$ is the $(1-\alpha)$-quantile of the standard normal distribution. Then, we construct a valid upper confidence bound on $R(\hat{f}|G_j)$ at level $1-\alpha$: $\mathbb{P}(R(\hat{f}|G_j) \leq \hat{L}_{j,u}(\alpha))\geq 1-\alpha$.

\paragraph{Hoeffding-based UCB}
When the loss function is bounded in $[0, B]$, we can replace the CLT-based UCB with a finite-sample bound using Hoeffding's inequality.
In line 16 of Algorithm~\ref{alg:ucb}, instead of computing the CLT-based bound, we first define $R = B / \pi_{\min}$ where $\pi_{\min} = \min_i \pi_i$, then compute:
\begin{equation*}
\delta_{\text{HB}}(\alpha) = \sqrt{\frac{R^2 \log(2/\alpha)}{2m}}, \quad \hat{L}_u(\alpha) = \hat{\mu}_Z(u) + \delta_{\text{HB}}(\alpha).
\end{equation*}
This bound provides strict finite-sample guarantees without relying on asymptotic normality, which is useful when the sample size $m$ is small.
The trade-off is that Hoeffding's bound is typically more conservative than the CLT-based bound for large $m$.

The CLT-based and Hoeffding-based approaches are two examples of constructing valid UCBs.
Alternatively, one could use other concentration inequalities such as Bernstein's inequality~\citep{bentkus2004hoeffdings,hao2019bootstrapping} to construct a valid confidence bound, which may provide better guarantees under different conditions.
Other approaches include betting-based confidence intervals and empirical Bernstein bounds.
The choice of UCB construction depends on the sample size, loss boundedness, and desired tightness.

\section{Proof of Theorem~\ref{thm:known-grouping}}
\label{app:proof-known}


\begin{proof}
Fix a group $G_j \in \mathcal{G}$ and let $\mathcal{D}_j = \{x_i \in \mathcal{D}_{\text{cal}} : x_i \in G_j\}$ be the calibration data for this group.
By Assumption~\ref{assump:ucb-validity}, the UCB $\hat{L}_{j,u}(\alpha)$ computed on $\mathcal{D}_j$ satisfies
\begin{equation}
\mathbb{P}_{\mathcal{D}_j}(R(u \mid G_j) \leq \hat{L}_{j,u}(\alpha)) \geq 1 - \alpha.
\end{equation}

The algorithm selects the threshold $\hat{u}_j = \max\{u : \hat{L}_{j,u}(\alpha) \leq \varepsilon\}$.
By definition, $\hat{L}_{j,\hat{u}_j}(\alpha) \leq \varepsilon$.
Combining with the UCB validity, with probability at least $1 - \alpha$:
\begin{equation}
R(\hat{f} \mid G_j) = R(\hat{u}_j \mid G_j) \leq \hat{L}_{j,\hat{u}_j}(\alpha) \leq \varepsilon.
\end{equation}

If no valid threshold exists (i.e., $\hat{L}_{j,u}(\alpha) > \varepsilon$ for all $u$), we set $\hat{u}_j = -\infty$, which means $\hat{f}(x) = f(x)$ for all $x \in G_j$.
In this case, $R(\hat{f} \mid G_j) = 0 \leq \varepsilon$ trivially.
Since each group $G_j$ uses disjoint calibration data $\mathcal{D}_j$, the calibration procedures are independent.
The guarantee $\mathbb{P}_{\mathcal{D}_j}(R(\hat{f} \mid G_j) \leq \varepsilon) \geq 1 - \alpha$ holds for each group separately.
No Bonferroni correction is needed because we do not require simultaneous validity across all groups.
\end{proof}

\section{Proof of Theorem~\ref{thm:empirical-known-grouping}}
\label{app:proof-empirical-known}


\begin{proof}
The proof combines the population risk guarantee from Theorem~\ref{thm:known-grouping} with Hoeffding's inequality for the test set.
Let $\hat{u}_j$ be the threshold selected by Algorithm~\ref{alg:calibration} for group $G_j$, which is determined by the calibration set $\mathcal{D}_j$.
Conditioned on $\hat{u}_j$, the test losses $\ell(\hat{f}(x_i), f(x_i))$ for $x_i \in G_j \cap \mathcal{D}_{\text{test}}$ are i.i.d.\ and bounded in $[0, B]$.
By Hoeffding's inequality, for any $t > 0$:
\begin{equation}
\mathbb{P}\left(\hat{R}(\hat{f} \mid G_j) - R(\hat{f} \mid G_j) > t \mid \hat{u}_j\right) \leq \exp\left(-\frac{2 N_j t^2}{B^2}\right).
\end{equation}
Using the inclusion:
\begin{equation}
\{\hat{R}(\hat{f} \mid G_j) > \varepsilon + t\} \subseteq \{R(\hat{f} \mid G_j) > \varepsilon\} \cup \{\hat{R}(\hat{f} \mid G_j) - R(\hat{f} \mid G_j) > t\}.
\end{equation}

Taking probabilities and applying the union bound:
\begin{align}
\mathbb{P}(\hat{R}(\hat{f} \mid G_j) > \varepsilon + t) &\leq \mathbb{P}(R(\hat{f} \mid G_j) > \varepsilon) + \mathbb{P}(\hat{R}(\hat{f} \mid G_j) - R(\hat{f} \mid G_j) > t) \\
&\leq \alpha + \exp\left(-\frac{2 N_j t^2}{B^2}\right).
\end{align}

The first term is bounded by $\alpha$ from Theorem~\ref{thm:known-grouping}.
The second term follows from the law of total probability and the conditional Hoeffding bound.
Therefore:
\begin{equation}
\mathbb{P}\left(\hat{R}(\hat{f} \mid G_j) \leq \varepsilon + t\right) \geq 1 - \alpha - \exp\left(-\frac{2 N_j t^2}{B^2}\right).
\end{equation}
\end{proof}

\section{Proof of Proposition~\ref{prop:grouping-benefit}}
\label{app:proof-grouping-benefit}

We prove that grouping always improves PAC efficiency, with strict improvement when groups have heterogeneous risk functions.

\begin{proof}
Let $\mathcal{G} = \{G_1, \ldots, G_k\}$ be a partition of $\mathcal{X}$ with group probabilities $p_j = \mathbb{P}(x \in G_j)$.
Let $\hat{u}^*$ denote the marginal PAC-optimal threshold (no grouping), and let $\hat{u}_j^*$ denote the PAC-optimal threshold for group $G_j$.
By definition, the marginal threshold $\hat{u}^*$ satisfies the PAC constraint on the entire population:
\begin{equation*}
\mathbb{P}_{\mathcal{D}_{\text{cal}}}(R(\hat{u}^*) \leq \varepsilon) \geq 1 - \alpha.
\end{equation*}
The marginal risk decomposes as $R(u) = \sum_{j=1}^k p_j \cdot R(u \mid G_j)$.
Since $R(\hat{u}^*) \leq \varepsilon$ implies $\sum_{j=1}^k p_j \cdot R(\hat{u}^* \mid G_j) \leq \varepsilon$, the threshold $\hat{u}^*$ is feasible for the marginal problem.
For each group $G_j$, the group-conditional calibrated threshold is defined as
\begin{equation*}
\hat{u}_j^* = \sup\{u : \mathbb{P}_{\mathcal{D}_j}(R(u \mid G_j) \leq \varepsilon) \geq 1 - \alpha\}.
\end{equation*}
We claim that $\hat{u}_j^* \geq \hat{u}^*$ for all $j \in [k]$.
Note that $R(\hat{u}^*) \leq \varepsilon$ implies $\sum_{j=1}^k p_j \cdot R(\hat{u}^* \mid G_j) \leq \varepsilon$, but not necessarily $R(\hat{u}^* \mid G_j) \leq \varepsilon$ for each $j$.
Since $\hat{u}_j^*$ is optimized for $G_j$ alone while $\hat{u}^*$ satisfies $\sum_{j=1}^k p_j \cdot R(\hat{u}^* \mid G_j) \leq \varepsilon$, we have:
\begin{align*}
R(\hat{u}^* \mid G_j) < R(\hat{u}^*) \quad &\Rightarrow \quad \hat{u}_j^* > \hat{u}^*, \\
R(\hat{u}^* \mid G_j) > R(\hat{u}^*) \quad &\Rightarrow \quad \hat{u}_j^* < \hat{u}^*.
\end{align*}
The PAC efficiency under grouping is
\begin{equation*}
\text{Eff}_{\text{PAC}}(\mathcal{G}) = \sum_{j=1}^k p_j \cdot \mathbb{P}(U(x) \leq \hat{u}_j^* \mid x \in G_j).
\end{equation*}
The marginal PAC efficiency is
\begin{equation*}
\text{Eff}_{\text{PAC}}(\{\mathcal{X}\}) = \mathbb{P}(U(x) \leq \hat{u}^*) = \sum_{j=1}^k p_j \cdot \mathbb{P}(U(x) \leq \hat{u}^* \mid x \in G_j).
\end{equation*}

To show $\text{Eff}_{\text{PAC}}(\mathcal{G}) \geq \text{Eff}_{\text{PAC}}(\{\mathcal{X}\})$, note that $\hat{u}^*$ is feasible for the group-conditional problem.
For each $G_j$, $\hat{u}_j^*$ maximizes $\mathbb{P}(U(x) \leq u \mid x \in G_j)$ subject to $\mathbb{P}_{\mathcal{D}_j}(R(u \mid G_j) \leq \varepsilon) \geq 1 - \alpha$.
Thus
\begin{equation*}
\mathbb{P}(U(x) \leq \hat{u}_j^* \mid x \in G_j) \geq \mathbb{P}(U(x) \leq \hat{u}^* \mid x \in G_j).
\end{equation*}
Summing over all groups:
\begin{equation*}
\text{Eff}_{\text{PAC}}(\mathcal{G}) = \sum_{j=1}^k p_j \cdot \mathbb{P}(U(x) \leq \hat{u}_j^* \mid x \in G_j) \geq \sum_{j=1}^k p_j \cdot \mathbb{P}(U(x) \leq \hat{u}^* \mid x \in G_j) = \text{Eff}_{\text{PAC}}(\{\mathcal{X}\}).
\end{equation*}
Equality holds iff $\hat{u}_j^* = \hat{u}^*$ for all $j \in [k]$, which occurs when $R(u \mid G_j) = R(u \mid G_{j'})$ for all $j, j'$ (no heterogeneity).
When $\exists j \neq j': R(u \mid G_j) \neq R(u \mid G_{j'})$, we have $\exists j: \hat{u}_j^* \neq \hat{u}^*$, yielding strict improvement.
\end{proof}

\section{Proof of Theorem~\ref{thm:unknown-grouping}}
\label{app:proof-unknown}

We prove the PAC guarantee for unknown grouping under both sample splitting and joint approaches.

\begin{proof}
We prove the two parts of the theorem separately.

\subsection*{1. Sample splitting approach}
The proof has two components: (1) the coverage guarantee, and (2) the efficiency gap bound.

\textbf{Coverage guarantee.}
Let $\hat{\mathcal{G}} = \{\hat{G}_1, \ldots, \hat{G}_k\}$ be the partition learned from $\mathcal{D}_{\text{cluster}}$.
Since $\mathcal{D}_{\text{cluster}}$ and $\mathcal{D}_{\text{cal}}$ are disjoint, the partition $\hat{\mathcal{G}}$ is independent of the calibration data $\mathcal{D}_{\text{cal}}$.
Conditioned on $\hat{\mathcal{G}}$, the calibration procedure is identical to the known grouping case.
For each group $\hat{G}_j$, let $\mathcal{D}_j = \{x_i \in \mathcal{D}_{\text{cal}} : x_i \in \hat{G}_j\}$ be the calibration data for this group.
By Assumption~\ref{assump:ucb-validity}, the UCB $\hat{L}_{j,u}(\alpha)$ computed on $\mathcal{D}_j$ satisfies:
\begin{equation*}
\mathbb{P}_{\mathcal{D}_j}(R(u \mid \hat{G}_j) \leq \hat{L}_{j,u}(\alpha) \mid \hat{\mathcal{G}}) \geq 1 - \alpha.
\end{equation*}
The algorithm selects $\hat{u}_j = \max\{u : \hat{L}_{j,u}(\alpha) \leq \varepsilon\}$, so $\hat{L}_{j,\hat{u}_j}(\alpha) \leq \varepsilon$.
Combining with the UCB validity:
\begin{equation*}
\mathbb{P}_{\mathcal{D}_j}(R(\hat{f} \mid \hat{G}_j) \leq \varepsilon \mid \hat{\mathcal{G}}) \geq 1 - \alpha.
\end{equation*}
Since this holds for any fixed $\hat{\mathcal{G}}$, by the law of total probability:
\begin{equation*}
\mathbb{P}_{\mathcal{D}_{\text{cal}}}(R(\hat{f} \mid \hat{G}_j) \leq \varepsilon) \geq 1 - \alpha.
\end{equation*}

\textbf{Efficiency gap bound.}
Let $\delta = d(\hat{\mathcal{G}}, \mathcal{G}^*)$ be the partition gap.
For any input $x$, let $g(x)$ and $g^*(x)$ denote the group assignments under $\hat{\mathcal{G}}$ and $\mathcal{G}^*$ respectively.
The efficiency difference can be bounded as:
\begin{align*}
&\text{Eff}_{\text{PAC}}(\mathcal{G}^*; \varepsilon, \alpha) - \text{Eff}_{\text{PAC}}(\hat{\mathcal{G}}; \varepsilon, \alpha) \\
&= \mathbb{E}[\mathbf{1}\{U(x) \leq \hat{u}_{g^*(x)}^*\}] - \mathbb{E}[\mathbf{1}\{U(x) \leq \hat{u}_{g(x)}^*\}] \\
&\leq \mathbb{P}(g(x) \neq g^*(x)) + \mathbb{P}(g(x) \neq g^*(x)) \\
&= 2\delta.
\end{align*}
The inequality follows because the efficiency difference is at most 1 for misclassified inputs and 0 for correctly classified inputs.

\subsection*{2. Joint approach}
Let $\hat{\mathcal{G}} = \{\hat{G}_1, \ldots, \hat{G}_k\}$ be the learned partition from calibration set $\mathcal{D}_{\text{cal}} = \{(x_i, U_i, L_i)\}_{i=1}^n$.
Let $\mathcal{G}^* = \{G_1^*, \ldots, G_k^*\}$ be the oracle optimal partition, and let $\delta = d(\hat{\mathcal{G}}, \mathcal{G}^*)$ be the partition gap.
We prove that with probability at least $1 - \alpha$, all learned groups satisfy $R(\hat{f} \mid \hat{G}_j) \leq \varepsilon + c \cdot \delta$.

For each learned group $\hat{G}_j$, the population risk can be decomposed as:
\begin{equation*}
R(\hat{f} \mid \hat{G}_j) = \mathbb{E}_{x \sim \mathcal{P}}[\ell(\hat{f}(x), f(x)) \mid x \in \hat{G}_j].
\end{equation*}
Let $\hat{g}: \mathcal{X} \to [k]$ and $g^*: \mathcal{X} \to [k]$ denote the group assignment functions for $\hat{\mathcal{G}}$ and $\mathcal{G}^*$ respectively.
We decompose the population risk based on whether the learned assignment matches the oracle:
\begin{align*}
R(\hat{f} \mid \hat{G}_j) &= \mathbb{E}[\ell(\hat{f}(x), f(x)) \mid \hat{g}(x) = j] \\
&= \mathbb{E}[\ell(\hat{f}(x), f(x)) \cdot \mathbf{1}\{\hat{g}(x) = g^*(x)\} \mid \hat{g}(x) = j] \\
&\quad + \mathbb{E}[\ell(\hat{f}(x), f(x)) \cdot \mathbf{1}\{\hat{g}(x) \neq g^*(x)\} \mid \hat{g}(x) = j].
\end{align*}

For the misclassification term, since the loss function is bounded by $B$, we have:
\begin{align*}
\mathbb{E}[\ell(\hat{f}(x), f(x)) \cdot \mathbf{1}\{\hat{g}(x) \neq g^*(x)\} \mid \hat{g}(x) = j] &\leq B \cdot \mathbb{P}(\hat{g}(x) \neq g^*(x) \mid \hat{g}(x) = j) \\
&\leq B \cdot \frac{\mathbb{P}(\hat{g}(x) \neq g^*(x))}{\mathbb{P}(\hat{g}(x) = j)} = B \cdot \frac{\delta}{p_j},
\end{align*}
where $p_j = \mathbb{P}(\hat{g}(x) = j)$ is the probability mass of learned group $\hat{G}_j$.

For inputs with $\hat{g}(x) = g^*(x)$, the learned group assignment matches the oracle.
Conditioned on the learned partition $\hat{\mathcal{G}}$, define the empirical risk on calibration data:
\begin{equation*}
\hat{R}_j = \frac{1}{|\mathcal{D}_j|} \sum_{i: x_i \in \hat{G}_j} L_i \cdot \mathbf{1}\{U_i \leq \hat{u}_j\},
\end{equation*}
where $\mathcal{D}_j = \{x_i \in \mathcal{D}_{\text{cal}} : x_i \in \hat{G}_j\}$.
By Assumption~\ref{assump:ucb-validity}, the UCB $\hat{L}_{j,u}(\alpha)$ satisfies:
\begin{equation*}
\mathbb{P}_{\mathcal{D}_j}(R(\hat{u}_j \mid \hat{G}_j, \hat{g} = g^*) \leq \hat{L}_{j,\hat{u}_j}(\alpha) \mid \hat{\mathcal{G}}) \geq 1 - \alpha,
\end{equation*}
where $R(\hat{u}_j \mid \hat{G}_j, \hat{g} = g^*)$ denotes the risk restricted to correctly classified inputs.
Since the algorithm selects $\hat{u}_j = \max\{u : \hat{L}_{j,u}(\alpha) \leq \varepsilon\}$, we have $\hat{L}_{j,\hat{u}_j}(\alpha) \leq \varepsilon$.
Thus, with probability at least $1 - \alpha$:
\begin{equation*}
\mathbb{E}[\ell(\hat{f}(x), f(x)) \cdot \mathbf{1}\{\hat{g}(x) = g^*(x)\} \mid \hat{g}(x) = j] \leq \varepsilon.
\end{equation*}

Combining the bounds for both terms, for each group $\hat{G}_j$, with probability at least $1 - \alpha$:
\begin{equation*}
R(\hat{f} \mid \hat{G}_j) \leq \varepsilon + B \cdot \frac{\delta}{p_j}.
\end{equation*}
To simplify the bound, we note that $p_j \geq p_{\min}$ where $p_{\min} = \min_j \mathbb{P}(\hat{g}(x) = j)$.
For balanced partitions with $k$ groups, $p_{\min} \geq 1/(2k)$ with high probability.
Setting $c = 2Bk$, we obtain for each group $\hat{G}_j$:
\begin{equation*}
\mathbb{P}(R(\hat{f} \mid \hat{G}_j) \leq \varepsilon + c \cdot \delta) \geq 1 - \alpha.
\end{equation*}
This guarantee holds for each group independently with its own confidence $1 - \alpha$.
\end{proof}



\section{Experimental setup}
\label{app:exp_setup}

\subsection{LLMs and datasets}
\label{app:llm_datasets}

\paragraph{Large language models} In this study, we evaluate the G-PAC reasoning based on the Qwen3 series models~\citep{yang2025qwen3} and the  Llama-3.1-8B–based LLMs. 
Specifically, we employ the ``Qwen3-4B-Thinking-2507'' as the thinking model and ``Qwen3-4B-Instruct-2507'' as the lower-performance non-thinking model. 
We also conduct a complementary experiment on Llama-3.1-8B–based models where the ``DeepSeek-R1-Distill-Llama-8B'' as the thinking model and ``Llama-3.1-8B-Instruct'' as the lower-performance non-thinking model. The sampling temperature and other hyperparameters for both LLMs are configured following the settings in the original paper. 
Experiments were run on four NVIDIA RTX A6000 Graphics Cards.

\paragraph{Datasets} We evaluate PAC reasoning on a range of real-world datasets spanning different reasoning paradigms.
Specifically, our evaluation includes a mathematical reasoning benchmark, MATH-500~\citep{lightman2023lets}, a text-based logical reasoning dataset, ZebraLogic~\citep{lin2025zebralogic}, and a commonsense question answering benchmark, GPQA~\citep{rein2024gpqa}.
For each dataset, we randomly split the original test set into a PAC calibration subset and a held-out PAC test subset.
Table~\ref{table:splitting_settings} summarizes the datasets used in our experiments, along with their corresponding splitting strategies, including dataset type, total size, and the sizes of the calibration and test partitions. We also report the distribution of categories for each dataset in Figure~\ref{fig:dataset_category_distribution}.

\begin{table}[ht]
\centering 
\caption{The details of datasets and splitting settings for PAC experiments}\label{table:splitting_settings}
\begin{tabular}{l|l|l|l|r}
\hline
\textbf{Dataset}& \textbf{Dataset Type} &\textbf{Dataset Size} & \textbf{Split Setting} & \textbf{Size} \\
\hline
\multirow{2}{*}{MATH-500} & \multirow{2}{*}{Math Reasoning} &\multirow{2}{*}{\centering 500}& PAC Calibration & 300 \\
     & & & PAC Test & 200 \\
\hline
\multirow{2}{*}{ZebraLogic} & \multirow{2}{*}{Text reasoning} &\multirow{2}{*}{\centering 1000}& PAC Calibration & 500 \\
    & & & PAC Test & 500 \\
\hline
\multirow{2}{*}{GPQA} & \multirow{2}{*}{QA Task} &\multirow{2}{*}{\centering 448}& PAC Calibration & 224 \\
     & & & PAC Test & 224 \\
\hline
\end{tabular}
\end{table}

\begin{figure}[t]
    \centering
    \begin{subfigure}[t]{0.31\linewidth}
        \centering
        \includegraphics[width=\linewidth]{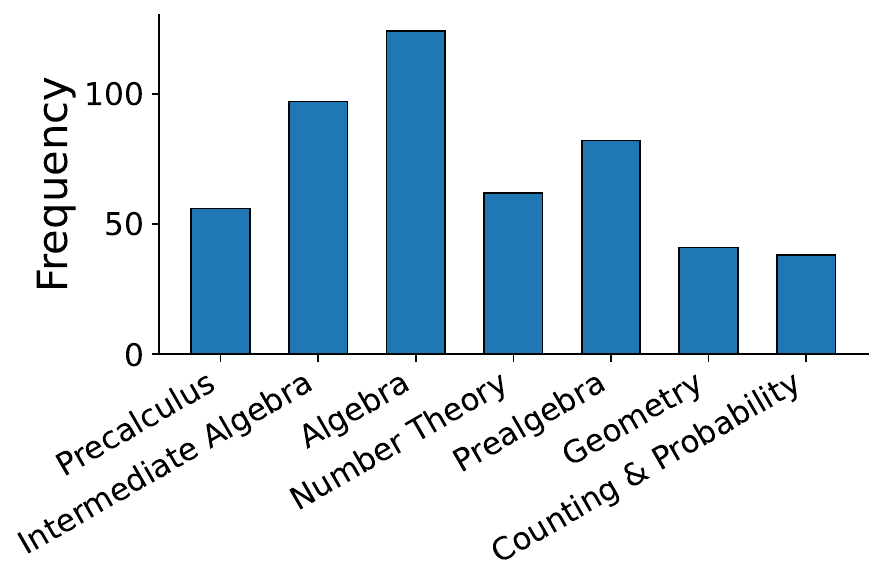}
        \caption{MATH-500}
    \end{subfigure}
    \begin{subfigure}[t]{0.31\linewidth}
        \centering
        \includegraphics[width=\linewidth]{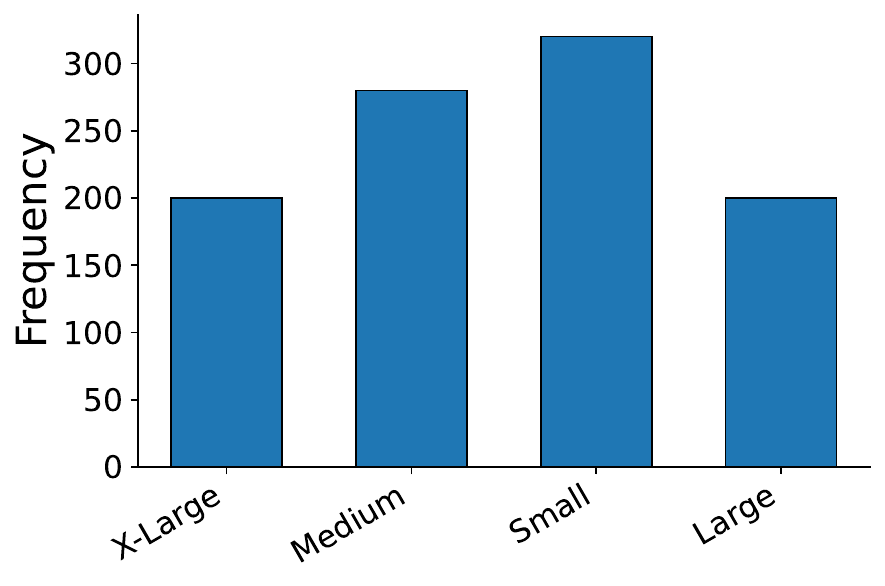}
        \caption{ZebraLogic}
    \end{subfigure}
     \begin{subfigure}[t]{0.31\linewidth}
        \centering
        \includegraphics[width=\linewidth]{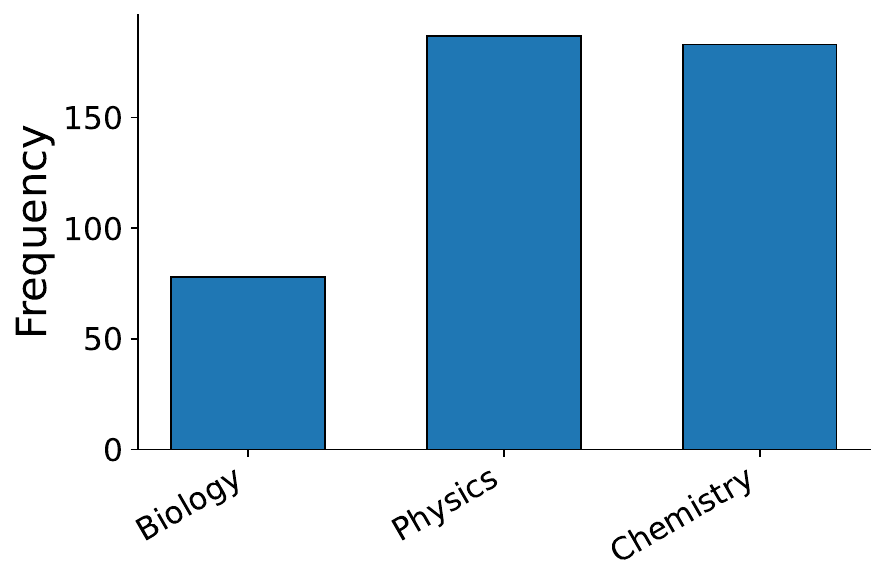}
        \caption{GPQA}
    \end{subfigure}
    \caption{
   The category distribution for three verfiable benchmarks.
    }
    \label{fig:dataset_category_distribution}
\end{figure}

\paragraph{The performance of tested LLMs on different benchmarks} We report the accuracy of the tested LLMs on different benchmarks in Table~\ref{table:llm_accuracy}. As shown, the performance gaps between models vary substantially across datasets, reflecting differences in task difficulty. Since such gaps directly influence the tightness of PAC-style guarantees, using a uniform tolerance would be either overly conservative for certain benchmarks. Therefore, we adopt dataset-specific values of $\varepsilon$ to appropriately account for these varying performance gaps and to ensure meaningful guarantees.

\begin{table*}[h]
\centering
\caption{Accuracy (\%) of different models on three reasoning benchmarks.}
\label{table:llm_accuracy}
\setlength{\tabcolsep}{2mm}{
\resizebox{0.8\textwidth}{!}{
\begin{tabular}{ccccc}
\toprule
\multirow{2}{*}{Dataset} 
& \multicolumn{2}{c}{Qwen LLMs} & \multicolumn{2}{c}{Llama-based LLMs}\\
\cmidrule(r){2-3} \cmidrule(r){4-5}
& Qwen3-4B-Instruct-2507 
& Qwen3-4B-Thinking-2507 
& Llama-3.1-8B-Instruct 
& DeepSeek-R1-Distill-Llama-8B \\
\midrule
MATH-500   & 92.80 & 96.20 &  45.20 & 85.80 \\
ZebraLogic & 80.40 & 89.20 & 12.50  & 38.40   \\
GPQA & 46.43 & 48.88 & 27.23 & 39.73 \\
\bottomrule
\end{tabular}}}
\end{table*}

\subsection{Uncertainty score function}
\label{app:uncertainty_score}
In this part, we introduce the details of three uncertainty score functions used to quantify the uncertainty of an answer from the non-thinking model, including a white-box score derived from model logits and two black-box scores obtained from verbalized self-reports and an external router model. 

\paragraph{Logits-based score} For the logits-based score, we use token-level probabilities computed from the prediction logits~\citep{kwon2023efficient,zheng2024sglang,zhou2025theoretical}.
Formally, let $y_i = (y_{i,1},\dots, y_{i,l})$ be an answer with $l$ tokens and $y_{i,j}$ be the $j$-th token of the answer. 
Furthermore, we define the uncertainty score of $y_i$ as its average token probability~\citep{hao2023reasoning, huang2025look}:
$$U_{logits}(y_i) = 1-\frac{1}{l_{y_i}} \sum_{j=1}^{l_{y_i}} \mathbb{P}(y_{i,j} | y_{i,1}, \dots, y_{i,j-1}, x_i),
$$
where $\mathbb{P}(y_{i,j} | y_{i,1}, \dots, y_{i,j-1},  x_i)$ is the conditional probability of token $y_{i,j}$.

\paragraph{Verbalized score} We also consider verbalized uncertainty scores from non-thinking models~\citep{xiong2023can,tian2023just,yang2024verbalized,zhou2025theoretical}, where the model explicitly states its self-reported confidence. The verbalized uncertainty score is mainly applicable in black-box scenarios, where access to generation logits is restricted, especially in the case of closed-source LLMs. The corresponding prompts are listed in Table~\ref{tab:verbalized_score_prompt}.  
In this study, we compute the average confidence over 10 trials and then define the verbalized uncertainty score as one minus this average confidence.

\begin{table*}[h]
\caption{Prompt for the verbalized confidence scores.}\label{tab:verbalized_score_prompt}
\renewcommand{\arraystretch}{1.4}
\rowcolors{1}{Gray}{Gray}
\resizebox{\textwidth}{!}{
\begin{tabular}{!{\vrule width 1.2pt} p{\linewidth}!{\vrule width 1.2pt}}
    \Xhline{1.2pt}
        \textbf{System prompt:} You are a reasoning assistant. For each question and proposed answer, you must estimate how likely the proposed answer is correct.\\
        
        \textbf{User prompt:}\\
        Question: \{QUESTION\}\\
        Answer: \{ANSWER\}\\
        Provide a probability (between 0.0 and 1.0) that your answer is correct. Only output the probability.\\
    \Xhline{1.2pt}
\end{tabular}}
\end{table*}

\paragraph{Router-based score} We trained a router for each mode pair based on an open-source library, named LLMRouter~\citep{2025llmrouter}. 
Specifically, we sample 2,000 examples from eight common benchmarks, including ARC-Challenge~\citep{clark2018think}, CommonsenseQA~\citep{talmor2019commonsenseqa}, GSM8K~\citep{cobbe2021training}, MATH~\citep{hendrycks2021measuring}, HumanEval~\citep{chen2021evaluating}, MMLU~\citep{hendrycks2021measuringa}, NaturalQA~\citep{kwiatkowski2019natural}, and TriviaQA~\citep{joshi2017triviaqa}. The router is trained as a lightweight classifier that takes the input prompt as features and predicts a routing score representing the probability that using the thinking model is necessary.
At inference time for our PAC reasoning, we define the uncertainty score of the non-thinking model as one minus the routing probability predicted for the non-thinking model.
Compared to other uncertainty scores, the router-based score relies on supervised signals collected from paired model outputs and can capture more complex input interactions between different models beyond simple confidence estimation. In this work, we use the ``Qwen3-Embedding-8B''~\citep{zhang2025qwen3} as the embedding model to encode the input prompts for training and inference of the router.

\begin{remark}
Logits-based and verbalized uncertainty scores follow a cascade routing setup, where a lightweight model first generates a candidate answer, and the corresponding uncertainty score is computed based on this output to decide whether to use a stronger model.
In contrast, router-based scores determine the routing decision directly from the router’s output, without requiring the non-thinking model to perform explicit answer generation.
Both setups are widely used in LLM routing~\citep{zeng2025pac,dekoninck2025unified,zhang2025leveraging}.
Notably, router-based routing is often more resource-efficient, as it avoids the cost of generating candidate answers with the non-thinking LLM. By supporting both answer-dependent and direct routing scores, our framework demonstrates its generality and compatibility with common routing strategies in practice.
\end{remark}

\paragraph{Alternative scores} The choice of the uncertainty score is flexible. Beyond the uncertainty scores considered in this work, a reward model~\citep{zhang2025lessons} can also be used to quantify the quality of the generated answer, and its output can naturally serve as an alternative uncertainty score.

\paragraph{Expected Calibration Error of different uncertainty scores}  Figure~\ref{fig:expected_error_uncertainty} reports the Expected Calibration Error (ECE) of different uncertainty scores across three verifiable benchmarks.
We observe substantial variation in calibration quality across datasets and scoring methods: logits-based scores are generally better calibrated than verbalized scores, while router-based scores often exhibit the largest ECE, especially on more challenging benchmarks.
Despite these differences, PAC reasoning consistently maintains valid control of performance loss across all settings (See Section~\ref{sec:experiments}).
This result highlights a key property of our framework: PAC reasoning does not rely on well-calibrated uncertainty estimates.
\textbf{These findings demonstrate that PAC reasoning is model-agnostic, dataset-agnostic, and score-agnostic.}

\begin{figure}[htbp]
    \centering

    \begin{subfigure}{0.31\linewidth}
        \centering
        \includegraphics[width=\linewidth]{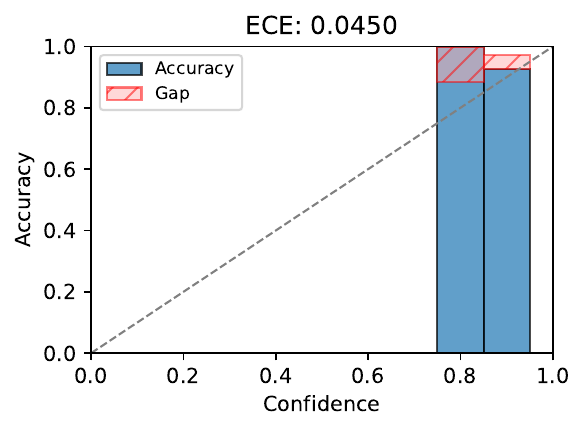}
        \caption{MATH-500 (Logits-based)}
    \end{subfigure}
    \begin{subfigure}{0.31\linewidth}
        \centering
        \includegraphics[width=\linewidth]{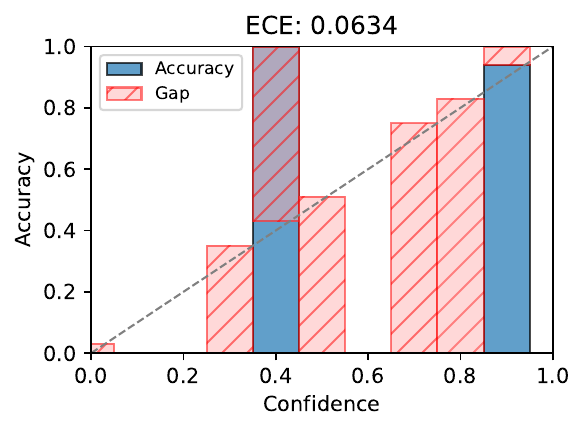}
        \caption{MATH-500 (Verbalized)}
    \end{subfigure}
    \begin{subfigure}{0.31\linewidth}
        \centering
        \includegraphics[width=\linewidth]{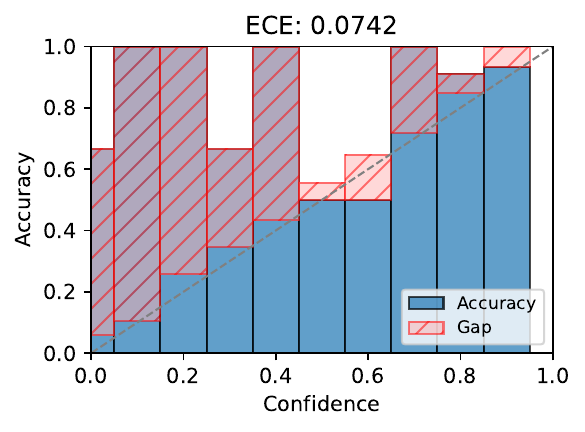}
        \caption{MATH-500 (Router)}
    \end{subfigure}

    \vspace{0.5em}

    \begin{subfigure}{0.31\linewidth}
        \centering
        \includegraphics[width=\linewidth]{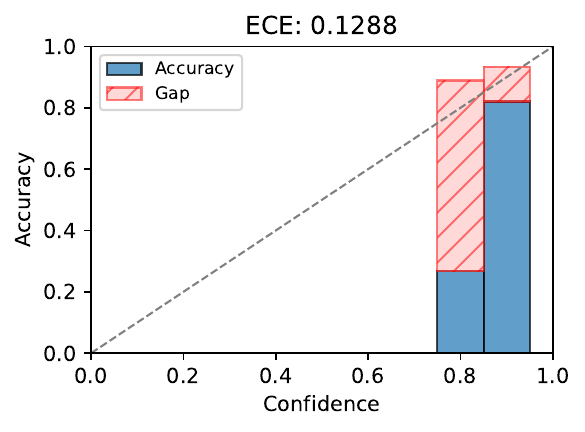}
        \caption{ZebraLogic (Logits-based)}
    \end{subfigure}
    \begin{subfigure}{0.31\linewidth}
        \centering
        \includegraphics[width=\linewidth]{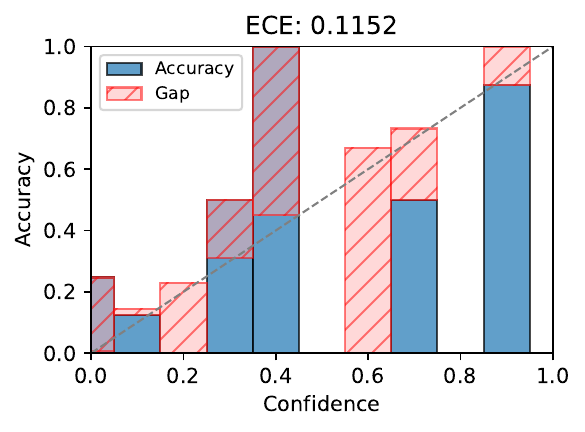}
        \caption{ZebraLogic (Verbalized)}
    \end{subfigure}
    \begin{subfigure}{0.31\linewidth}
        \centering
        \includegraphics[width=\linewidth]{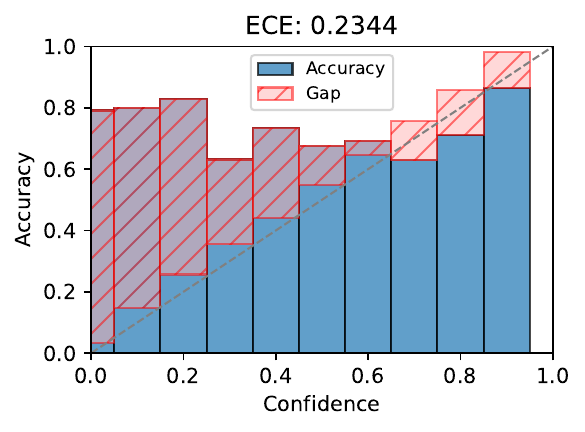}
        \caption{ZebraLogic (Router)}
    \end{subfigure}

    \vspace{0.5em}

    \begin{subfigure}{0.31\linewidth}
        \centering
        \includegraphics[width=\linewidth]{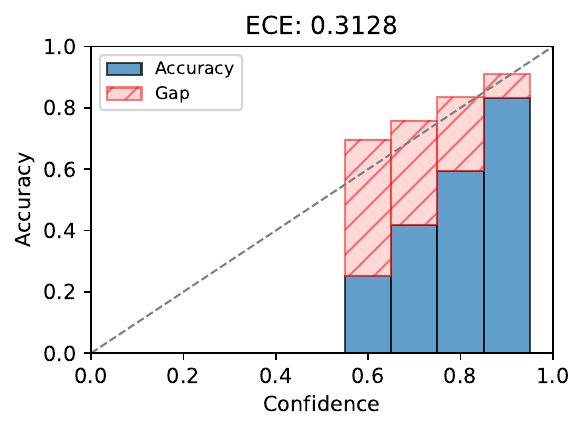}
        \caption{GPQA (Logits-based)}
    \end{subfigure}
    \begin{subfigure}{0.31\linewidth}
        \centering
        \includegraphics[width=\linewidth]{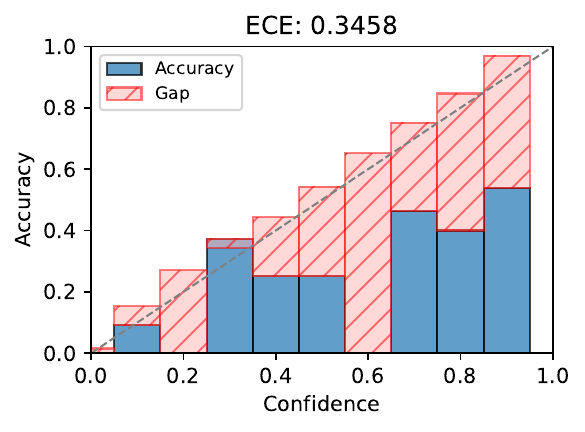}
        \caption{GPQA (Verbalized)}
    \end{subfigure}
    \begin{subfigure}{0.31\linewidth}
        \centering
        \includegraphics[width=\linewidth]{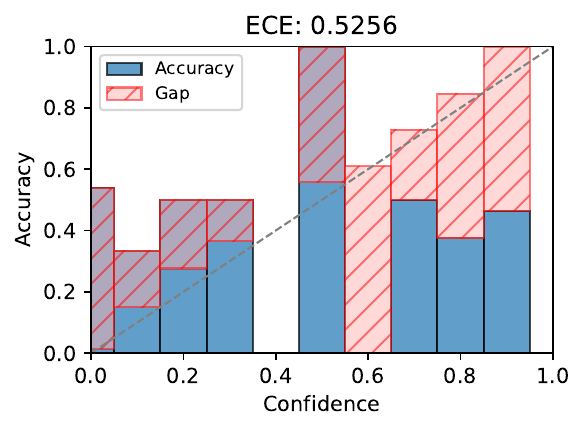}
        \caption{GPQA (Router)}
    \end{subfigure}

    \caption{Expected Calibration Error across three verifiable benchmarks using Qwen LLMs.}
    \label{fig:expected_error_uncertainty}
\end{figure}

\subsection{Evaluation metrics}
\label{app:evalution_metrics}
Let $T$ denote the number of independent trials, indexed by $t = 1,\dots,T$. To evaluate the effectiveness of our method, we consider the following metrics.

\paragraph{Risk evaluation}
For the $t$-th trial, the empirical error is defined as
\begin{align*}
\text{Error}^{(t)}
&=
\frac{1}{N}
\sum_{j=1}^k
\sum_{i \in \mathcal{I}_{\mathrm{test},j}^{(t)}}
\ell(y_i, \tilde{y}_i)\,
\mathbf{1}\{U(x_i) \le u\}.
\end{align*}

We first average the group-wise empirical risk across trials.
For each group $j$, define
\[
\bar{E}_j
:=
\frac{1}{T}
\sum_{t=1}^T
E_j^{(t)},
\quad
E_j^{(t)}
=
\frac{1}{N_j^{(t)}}
\sum_{i \in \mathcal{I}_{\mathrm{test},j}^{(t)}}
\ell(y_i, \tilde{y}_i)\,
\mathbf{1}\{U(x_i) \le u\}.
\]

We report the final metrics by averaging over all trials:
\[
\text{Error}
=
\frac{1}{T}
\sum_{t=1}^T
\text{Error}^{(t)},
\text{Error}_{\text{Gap}}
=
\sum_{j=1}^k
(\bar{E}_j - \varepsilon)\,
\mathbf{1}\{\bar{E}_j > \varepsilon\}.
\]
When $\text{Error}_{\mathrm{Gap}} = 0$, the group-conditional empirical risk
does not exceed the tolerance level $\varepsilon$ for any group in any trial.

\paragraph{Efficiency evaluation}
To evaluate the effectiveness of PAC reasoning in reducing inference cost, we introduce an efficiency metric termed \textit{Saved Token Percentage} (STP).
Let $l_{\tilde{y}_i}$ denote the number of tokens generated by the non-thinking model for its candidate answer $\tilde{y}_i$, and let $l_{y_i}$ denote the number of tokens generated by the thinking model for its reference answer $y_i$.
The STP is formally defined as
\begin{equation*}
 \text{STP} := \frac{1}{N} \sum_{i \in \mathcal{I}_{\text{test}}}
 \left(1 - \frac{l_{\tilde{y}_i} + \mathbf{1}\{U_i > u\} \, l_{y_i}}{l_{y_i}}\right)\times 100\% .
\end{equation*}
Here, $\mathbf{1}\{\cdot\}$ is the indicator function, $U_i$ denotes the uncertainty score, and $u$ is the calibrated threshold.
This metric measures the relative token savings by comparing the tokens consumed by PAC reasoning against always invoking the thinking model, while accounting for cases in which an expert (thinking-model) call is triggered.

In the router-based setting, however, the uncertainty score $U_i$ is produced by an external router rather than by the non-thinking model itself.
As a result, the non-thinking model is only executed when the router decides not to use the thinking model.
Accordingly, we slightly revise the STP definition as follows:
\begin{equation*}
 \text{STP} := \frac{1}{N} \sum_{i \in \mathcal{I}_{\text{test}}}
 \left(1 - \frac{\mathbf{1}\{U_i \leq u\}\, l_{\tilde{y}_i}
 + \mathbf{1}\{U_i > u\}\, l_{y_i}}{l_{y_i}}\right)\times 100\% .
\end{equation*}

We repeat each experiment 100 times and report the mean and standard deviation of the budget savings. We fix $\alpha = 0.05$ throughout all experiments while varying $\varepsilon$, and set the sampling weight $\pi = \pi_i = 0.5$ for each \(i\in \mathcal{I}_{cal}\) and the sample size $m = n \times \frac{1}{\pi}$.

\subsection{Loss function}
\label{app:loss_function}
Since the motivation is that PAC reasoning is designed to control the additional error introduced by switching from a reliable but expensive thinking model to a cheaper non-thinking model, this loss is intentionally defined relative to the thinking model rather than directly with respect to the ground-truth accuracy. 
In this work, we consider two types of loss functions for evaluating the PAC guarantee of our method: the semantic cosine distance and the binary 0-1 loss.

\paragraph{Semantic loss function}
The semantic cosine distance measures similarity between outputs in the embedding space. Formally, given reference output $y_i = f(x_i)$ and PAC reasoning output $\hat{y}_i = \hat{f}(x_i)$, we compute their embeddings $v_{y_i}$ and $v_{\hat{y}_i}$, and define the loss as: 
\begin{equation}
\label{eq:semantic_loss}
    \ell(y_i, \hat{y}_i) = 1 - \frac{v_{y_i} \cdot v_{\hat{y}_i}}{\|v_{y_i}\| \|v_{\hat{y}_i}\|}.
\end{equation}
For the semantic embedding model, we adopt ``Qwen3-Embedding-4B''~\citep{zhang2025qwen3}.

\paragraph{Binary loss function}
We also employ a binary 0–1 loss, defined in Eq.~\eqref{eq:binary_loss}, which captures the actual loss in answer accuracy when comparing the PAC reasoning output $\hat{y}$ with the reference output $y$:
\begin{equation}
\label{eq:binary_loss}
\ell(y_i, \hat{y}_i) = \ell(y_i, \hat{y}_i | y_{i,gold}) = \mathbf{1} \{ \hat{y}_i \neq y_{i}^{gold}\} \mathbf{1}\{y_i=y_{i}^{gold}\}
\end{equation}
where $y_{i,gold}$ is the ground-truth answer for the problem $x_i$. 

\begin{remark}
    By counting an error only when the thinking model produces a correct answer but the PAC reasoning output does not, this loss isolates the incremental performance degradation caused by efficiency-oriented routing.
This formulation is particularly suitable for switching-based deployment scenarios, where the thinking model serves as a reliability anchor, and the primary objective is to ensure that efficiency gains do not incur excessive additional errors.
In this sense, PAC reasoning provides a statistically valid guarantee on relative performance loss with respect to a fixed thinking model, rather than a guarantee of absolute correctness with respect to ground truth.
\end{remark}

\section{Main experimental results}

\subsection{Main results of Qwen series models}
In this part, we present the results of the Qwen series LLMs.

\paragraph{Error analysis across categories for vanilla PAC reasoning and G-PAC reasoning.}
To examine the validity of group-conditional guarantees, we report error bars for both the overall population and each category under the Qwen series models, as shown in Figure~\ref{fig:errorbar_category}. The corresponding results have been summarized in Table~\ref{table:01_results_known}.
While vanilla PAC reasoning successfully controls the overall error below the target tolerance~$\varepsilon$, it fails to consistently satisfy the error constraint at the category level, with several categories exhibiting errors that exceed~$\varepsilon$. 
In contrast, G-PAC reasoning maintains valid error control not only for the overall dataset but also uniformly across all categories, demonstrating its effectiveness in enforcing group-conditional guarantees.

\paragraph{Experiment of C-PAC under known group partitions.}
Table~\ref{table:01_results_unknown} presents the experimental results comparing vanilla PAC and C-PAC under the setting where the group partition is known. We evaluate both methods on three verifiable reasoning benchmarks (MATH-500, ZebraLogic, and GPQA) with $\alpha=0.05$. Across all scoring strategies and under both disjoint and joint clustering, C-PAC consistently achieves lower empirical error than PAC. More importantly, C-PAC effectively drives the group-wise error gap $\text{Error}_{\text{Gap}}$ to zero in all cases. These results demonstrate that incorporating group-aware calibration enables C-PAC to provide stronger and more reliable guarantees across heterogeneous groups.

\paragraph{Ablation analysis on the calibration ratios.}
Figure~\ref{fig:results_calibration_ratio_other} presents additional results on ZebraLogic and GPQA.
Across different calibration ratios, both PAC and G-PAC maintain stable marginal error control.
However, the group-conditional error gap of vanilla PAC varies noticeably with the calibration ratio, while G-PAC consistently achieves near-zero error gaps across all uncertainty scores.
Moreover, the efficiency of G-PAC changes smoothly with the calibration ratio and remains comparable to that of PAC.
These results further demonstrate the robustness of G-PAC to the choice of calibration set size.

\begin{figure*}[t]
    \centering
    \begin{subfigure}[t]{\linewidth}
        \centering
        \includegraphics[width=0.31\linewidth]{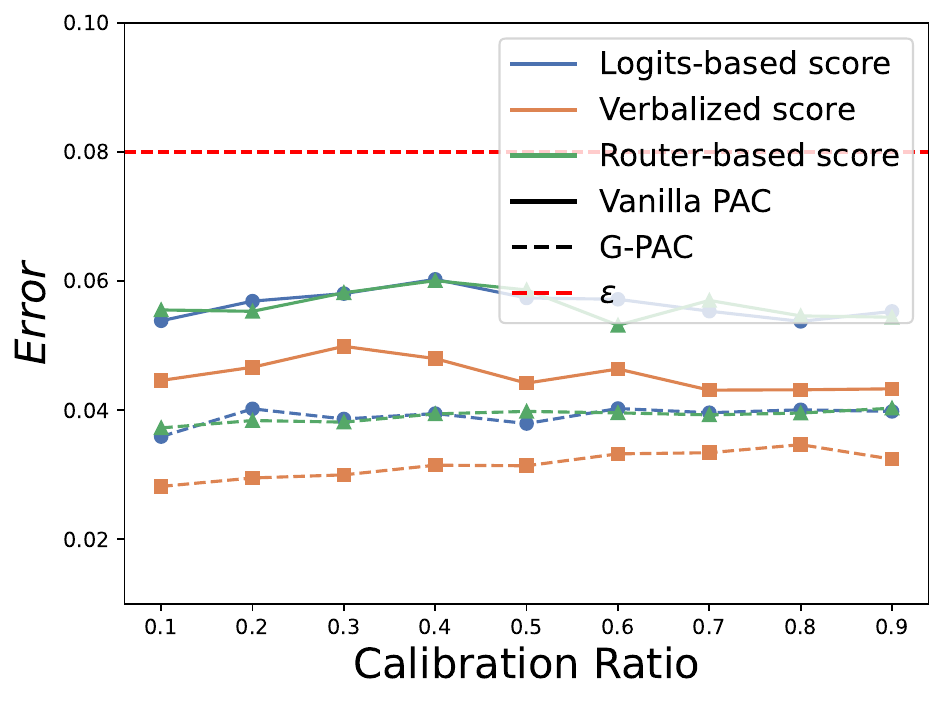}
        \includegraphics[width=0.31\linewidth]{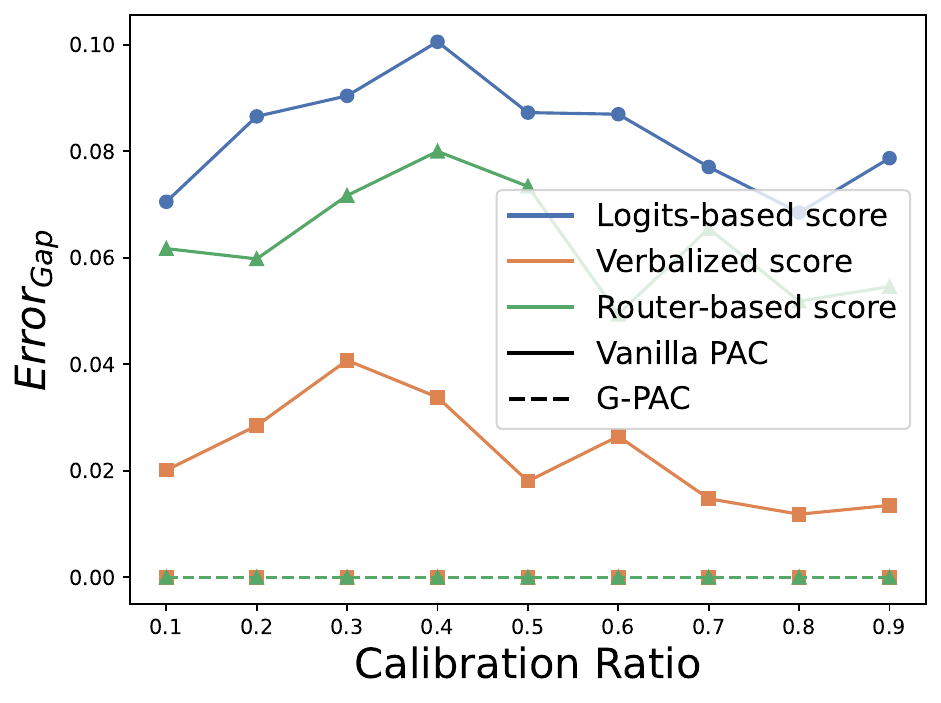}
        \includegraphics[width=0.31\linewidth]{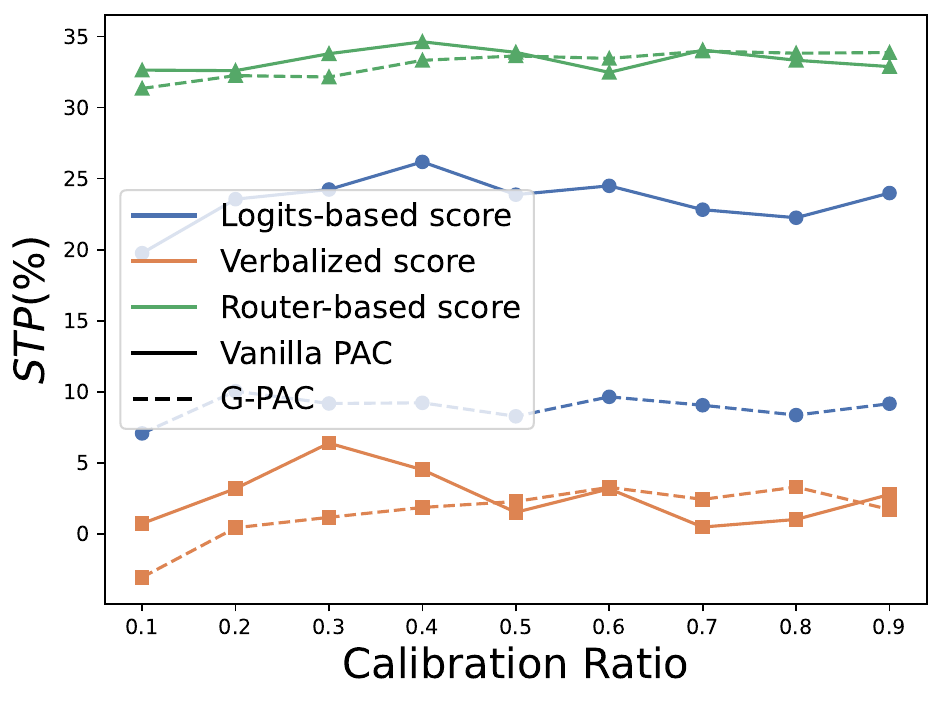}
        \caption{ZebraLogic}
    \end{subfigure}
    \hfill
    \begin{subfigure}[t]{\linewidth}
        \centering
        \includegraphics[width=0.31\linewidth]{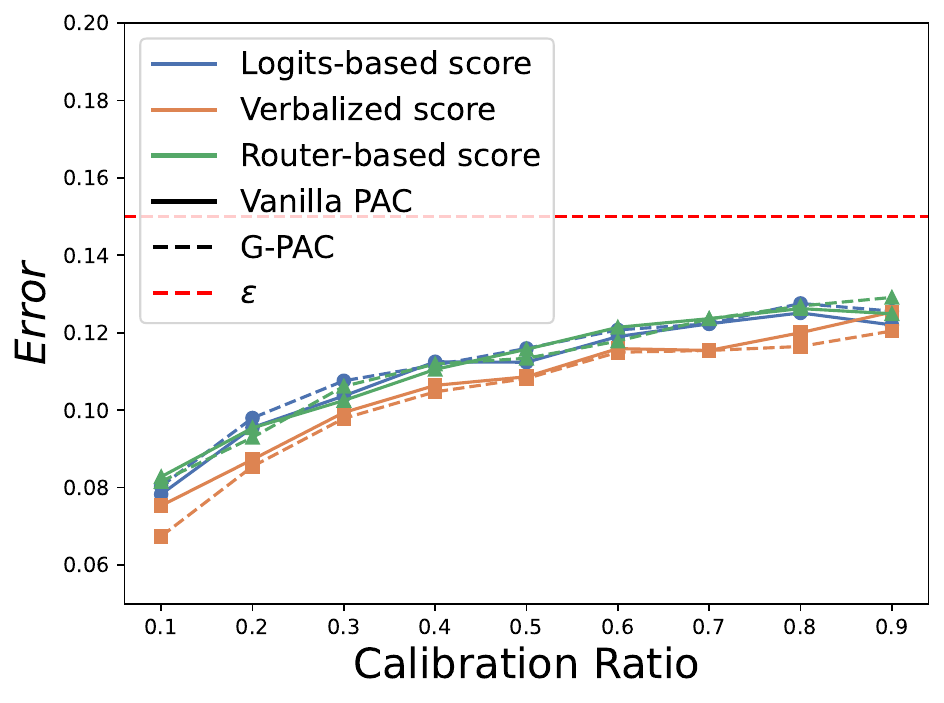}
        \includegraphics[width=0.31\linewidth]{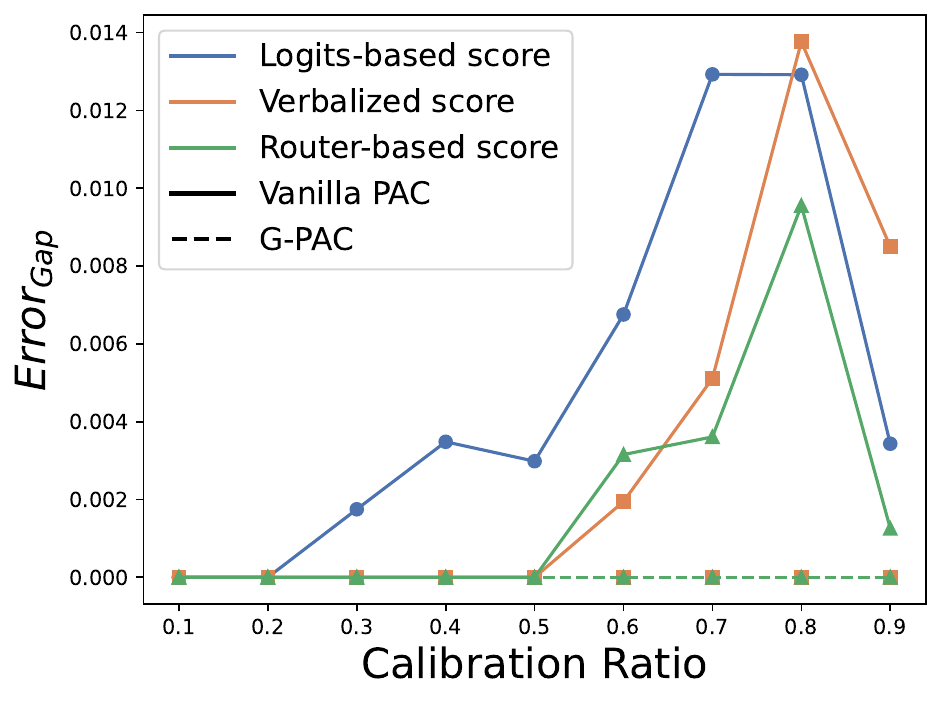}
        \includegraphics[width=0.31\linewidth]{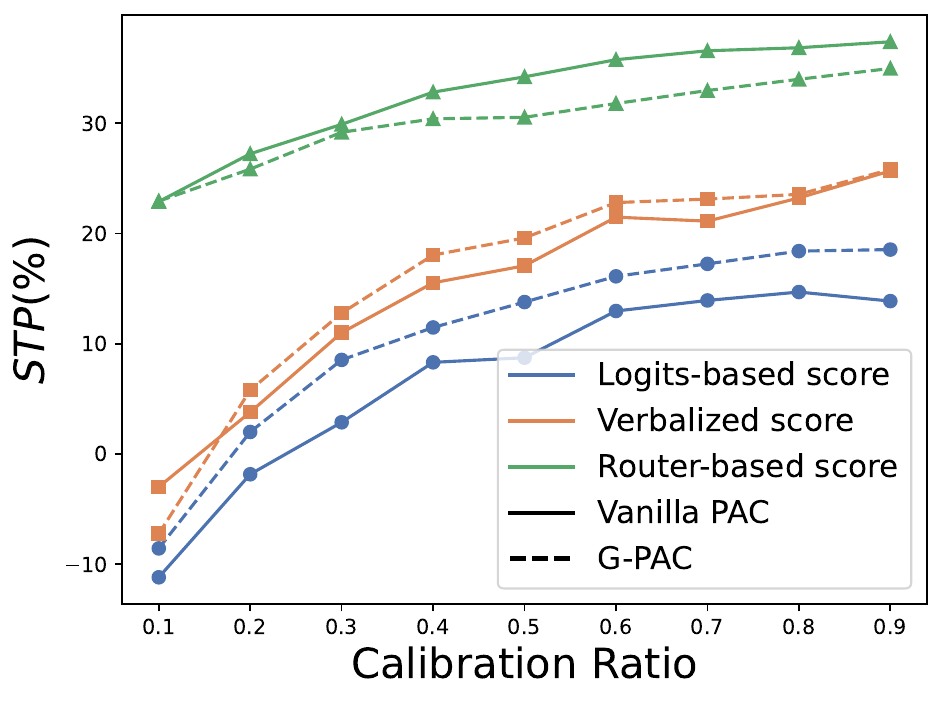}
        \caption{GPQA}
    \end{subfigure}
    \caption{
    Error control and STP of PAC reasoning for binary loss for different calibration ratios at a confidence level of $\alpha = 0.05$. All experiments are conducted on Qwen LLMs. The red dashed line $\varepsilon$ means the target risk level. 
    }
    \label{fig:results_calibration_ratio_other}
\end{figure*}

\subsection{Main results of Llama-based LLMs}
\label{app:exp_other_llm}

In this part, we report results for the Llama model.
Table~\ref{table:01_llama_results} shows that, when group partitions are known, G-PAC consistently eliminates the group-conditional error gap across all datasets and uncertainty scores on Llama-based models, while retaining comparable computational savings to PAC.
Moreover, Figure~\ref{fig:errorbar_category_llama} shows that the average risk for each category when using Llama-based LLMS.
Table~\ref{table:01_results_unknown_llama} demonstrates that under unknown group partitions, C-PAC maintains zero empirical $\text{Error}_{\text{Gap}}$ under both independent and joint clustering, whereas vanilla PAC exhibits substantial conditional risk violations.
Table~\ref{table:arena_results_llama} further confirms that on the open-domain Arena-Hard benchmark, C-PAC achieves strict conditional loss control with meaningful efficiency gains, even in the absence of verifiable answers and predefined group labels.

\begin{table*}[!t]
\centering
\caption{
\textbf{Performance comparison between vanilla PAC and G-PAC under known group partitions.}  
Experimental results of the binary loss function on verifiable datasets using the Qwen models with $\alpha=0.05$.  
We set $\varepsilon=0.05$ for MATH-500, $\varepsilon=0.1$ for ZebraLogic, and $\varepsilon=0.15$ for GPQA.}
\label{table:01_results_unknown}
\setlength{\tabcolsep}{3mm}{
\resizebox{\textwidth}{!}{
\begin{tabular}{lccccccccccccc}
\toprule
\multirow{3}{*}{Dataset}  & \multirow{3}{*}{Metric}  & \multicolumn{6}{c}{Disjoint clustering}& \multicolumn{6}{c}{Joint clustering} \\
\cmidrule(r){3-8} \cmidrule(r){9-14}
& & \multicolumn{2}{c}{Logits-based score} &\multicolumn{2}{c}{Verbalized score} & \multicolumn{2}{c}{Router-based score} & \multicolumn{2}{c}{Logits-based score} &\multicolumn{2}{c}{Verbalized score} & \multicolumn{2}{c}{Router-based score}
\\
\cmidrule(r){3-4} \cmidrule(r){5-6} \cmidrule(r){7-8} \cmidrule(r){9-10} \cmidrule(r){11-12} \cmidrule(r){13-14}
& & PAC & C-PAC & PAC & C-PAC & PAC & C-PAC  & PAC & C-PAC& PAC & C-PAC & PAC & C-PAC \\
\midrule

\multirow{3}{*}{MATH-500 }
& $\text{Error}$ (\%)
& 3.03 & 2.56 & 2.70 & 2.25 & 2.94 & 2.88
& 3.34 & 2.45 & 3.24 & 2.44 & 3.34 & 2.92 \\

& $\text{Error}_{\text{Gap}}$ (\%)
& 0.17 & \colorbox{green!10}{0.00} & 0.00 & \colorbox{green!10}{0.00} & 0.00 & \colorbox{green!10}{0.00}
& 1.02 & \colorbox{green!10}{0.00} & 0.33 &\colorbox{green!10}{0.00} & 0.14 & \colorbox{green!10}{0.00} \\

& STP (\%) $\uparrow$
& 50.95 & 43.19 & 47.47 & 35.01 & 54.10 & 53.04
& 58.34 & 45.79 & 60.48 & 42.73 & 59.29 & 54.51 \\

\midrule

\multirow{3}{*}{ZebraLogic}
& $\text{Error}$ (\%) 
& 7.69 & 5.41 & 6.67 & 5.99 & 7.36 & 6.77
& 7.51 & 5.45 & 6.67 & 5.97 & 7.77 & 6.74 \\

& $\text{Error}_{\text{Gap}}$ (\%) 
& 1.98 & \colorbox{green!10}{0.00} & 0.00 & \colorbox{green!10}{0.00} & 1.49 & \colorbox{green!10}{0.00}
& 1.44 & \colorbox{green!10}{0.00} & 0.00 & \colorbox{green!10}{0.00} & 2.20 & \colorbox{green!10}{0.00} \\

& STP (\%) $\uparrow$
& 36.07 & 18.99 & 20.68 & 10.87 & 39.01 & 35.38
& 35.72 & 18.95 & 21.06 & 10.55 & 40.34 & 35.34 \\

\midrule

\multirow{3}{*}{GPQA}
& $\text{Error}$ (\%)
& 11.57 & 11.21 & 11.24 & 9.46 & 11.79 & 11.36
& 11.27 & 11.57 & 11.00 & 8.87 & 11.91 & 11.60 \\

& $\text{Error}_{\text{Gap}}$ (\%)
& 0.00 & \colorbox{green!10}{0.00} & 0.06 & \colorbox{green!10}{0.00} & 1.68 & \colorbox{green!10}{0.00}
& 0.38 & \colorbox{green!10}{0.00} & 0.00 & \colorbox{green!10}{0.00} & 1.36 & \colorbox{green!10}{0.00} \\

& STP (\%) $\uparrow$
& 11.28 & 10.32 & 18.86 & 5.98 & 34.64 & 30.39
& 9.05 & 10.87 & 18.21 & 0.54 & 34.91 & 29.75 \\

\bottomrule
\end{tabular}}}
\end{table*}

\begin{table}[!t]
\centering
\caption{\textbf{C-PAC successfully controls the group-conditional risk on an open-domain task.}  
Experimental results on Arena-Hard with $\varepsilon=0.1$, using Qwen LLMs.}
\label{table:arena_results}
\small
\setlength{\tabcolsep}{6mm}
\resizebox{0.6\columnwidth}{!}{
\begin{tabular}{ccccc}
\toprule
Uncertainty & Method & Error (\%) & Error$_{\text{Gap}}$ (\%) & STP (\%) $\uparrow$ \\
\midrule
\multicolumn{5}{c}{Disjoint clustering} \\
\midrule
\multirow{2}{*}{Logits}
 & PAC   & 8.84 & 0.00 & 51.77 \\
 & C-PAC & 8.54 & \colorbox{green!10}{0.00} & 47.01 \\
\midrule
\multirow{2}{*}{Verbalized}
 & PAC   & 8.63 & 3.68 & 30.07 \\
 & C-PAC & 5.27 & \colorbox{green!10}{0.00} & 10.91 \\
\midrule
\multirow{2}{*}{Router}
 & PAC   & 8.90 & 4.44 & 33.07 \\
 & C-PAC & 8.84 & \colorbox{green!10}{0.00} & 42.34 \\
\midrule
\multicolumn{5}{c}{Joint clustering} \\
\midrule
\multirow{2}{*}{Logits}
 & PAC   & 8.94 & 0.00 & 52.25 \\
 & C-PAC & 8.69 & \colorbox{green!10}{0.00} & 49.32 \\
\midrule
\multirow{2}{*}{Verbalized}
 & PAC   & 8.67 & 4.04 & 30.15 \\
 & C-PAC & 5.17 & \colorbox{green!10}{0.00} & 10.30 \\
\midrule
\multirow{2}{*}{Router}
 & PAC   & 8.94 & 4.69 & 33.39 \\
 & C-PAC & 8.77 & \colorbox{green!10}{0.00} & 42.02 \\
\bottomrule
\end{tabular}}
\end{table}

\begin{table*}
\centering
\caption{ 
\textbf{G-PAC performs a smaller error gap than vanilla PAC.}
Experimental results on verifiable datasets, using Llama-based LLMs ($\alpha=0.05$).  
We set $\varepsilon=0.42$ for MATH-500, $\varepsilon=0.2$ for ZebraLogic, and $\varepsilon=0.2$ for GPQA.}\label{table:01_llama_results}
\setlength{\tabcolsep}{6mm}{
\resizebox{\textwidth}{!}{
\begin{tabular}{lccccccc}
\toprule
\multirow{2}{*}{Dataset}  & \multirow{2}{*}{Metric}  & \multicolumn{2}{c}{Logits-based score}& \multicolumn{2}{c}{Verbalized score} & \multicolumn{2}{c}{Router-based score}
\\
\cmidrule(r){3-4} \cmidrule(r){5-6} \cmidrule(r){7-8}
& & PAC & G-PAC & PAC & G-PAC & PAC & G-PAC  \\
\midrule
\multirow{3}{*}{MATH-500}
& $\text{Error}$ (\%)  & $ 37.28 $ & $32.16$ & $36.52$ & $ 32.18 $ & 37.17 &  32.99 \\
& $\text{Error}_{\text{Gap}}$ (\%)  & $  19.27 $ & \colorbox{green!10}{0.00} & $12.47$ & \colorbox{green!10}{0.00} & 16.20 & \colorbox{green!10}{0.00} \\
& STP (\%) $\uparrow$ & $12.55 $  & $2.91 $ & $11.32 $ & $-2.21 $ & 26.01 &23.04\\
\midrule
\multirow{3}{*}{ZebraLogic}
& $\text{Error}$ (\%) &  17.33 & 11.43 & 17.05 & 11.13 & 17.20 &  11.13 \\
& $\text{Error}_{\text{Gap}}$ (\%) & 16.20 & \colorbox{green!10}{0.00} & 15.37 & \colorbox{green!10}{0.00} & 15.55 & \colorbox{green!10}{0.00}\\
& STP (\%)  $\uparrow$& 23.72  &  13.23 &  -30.68 & -9.52 & 25.22 & 25.98 \\
\midrule
\multirow{3}{*}{GPQA}
& $\text{Error}$ (\%) & 12.38 & 12.40  & 11.44 & 11.91  & 11.89& 11.90 \\
& $\text{Error}_{\text{Gap}}$ (\%) & 7.03  & \colorbox{green!10}{0.00} & 0.00 & \colorbox{green!10}{0.00} & 0.00 & \colorbox{green!10}{0.00} \\
& STP (\%)  $\uparrow$& 18.13  &  23.47 & 42.52 & 36.87 &39.67 & 39.88\\
\bottomrule
\end{tabular}}}
\end{table*}

\begin{figure*}[t]
    \centering
    \begin{subfigure}[t]{\linewidth}
        \centering
        \includegraphics[width=0.27\linewidth]{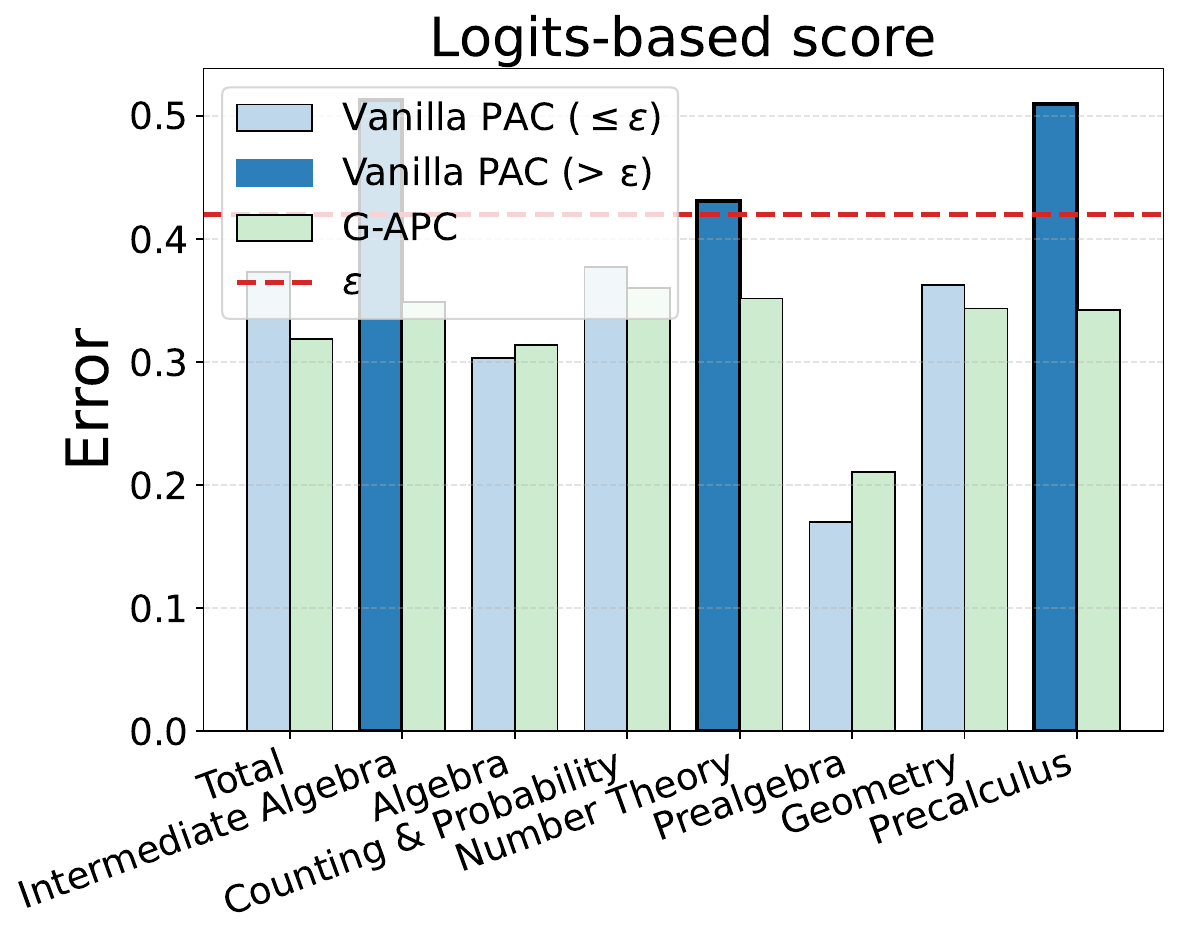}
        \includegraphics[width=0.27\linewidth]{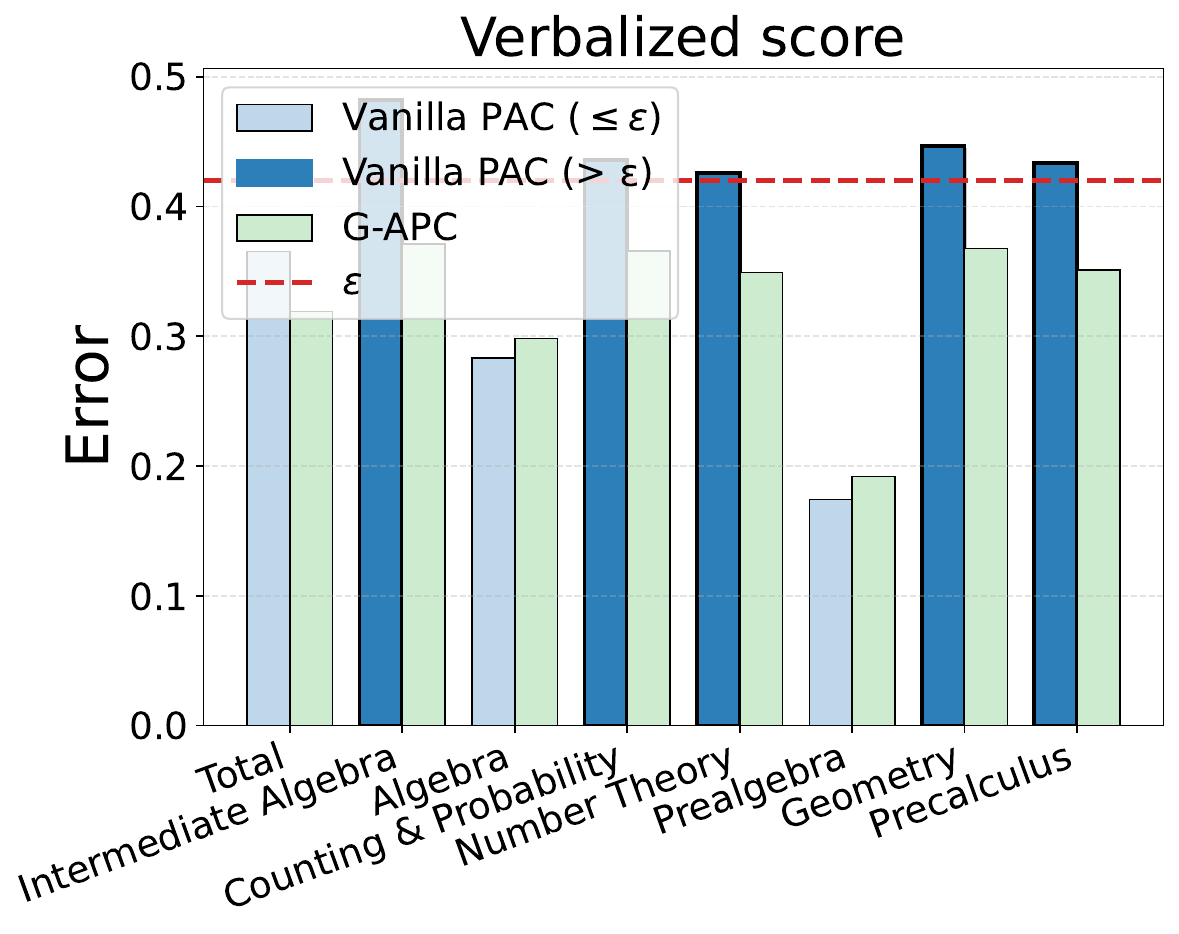}
        \includegraphics[width=0.27\linewidth]{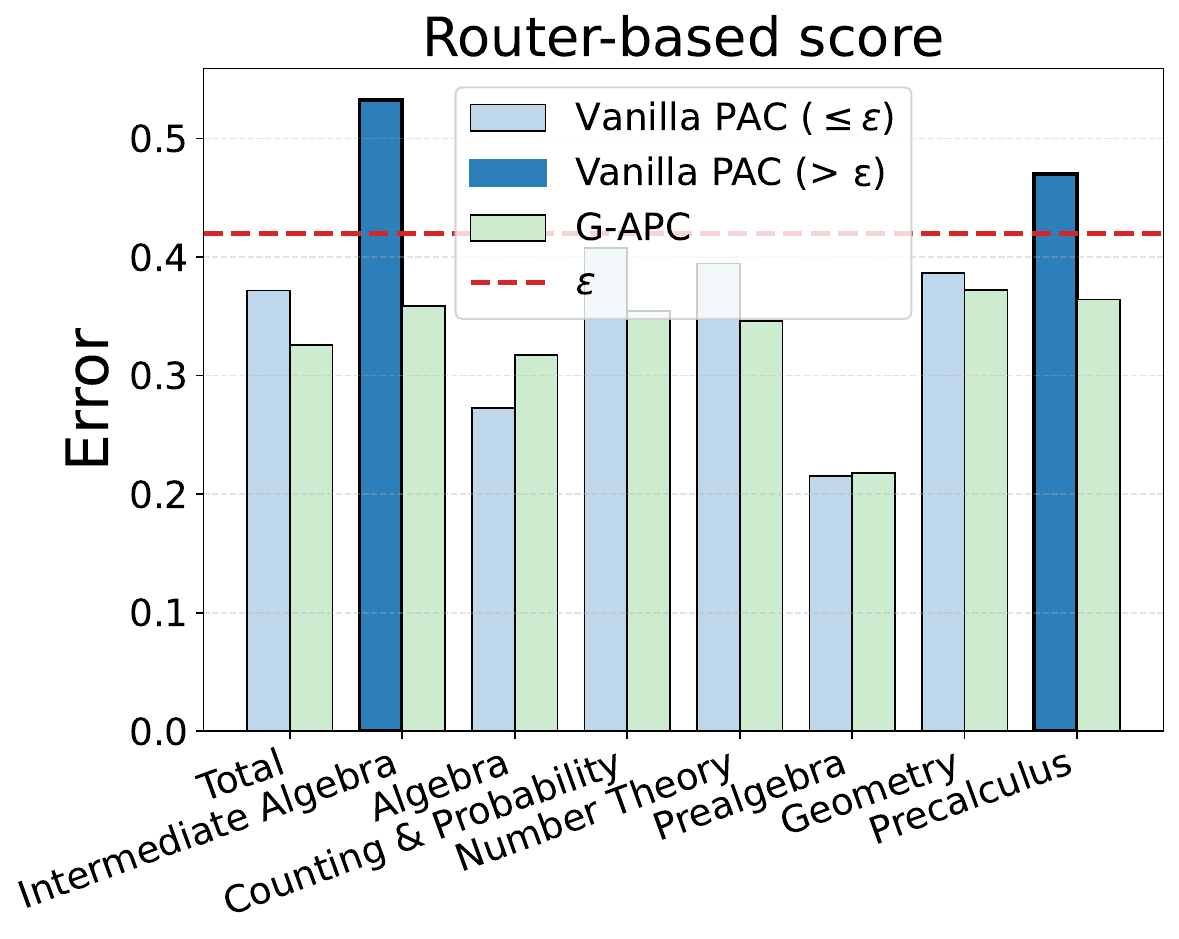}
        \caption*{MATH-500}
    \end{subfigure}\\
    \begin{subfigure}[t]{\linewidth}
        \centering
        \includegraphics[width=0.27\linewidth]{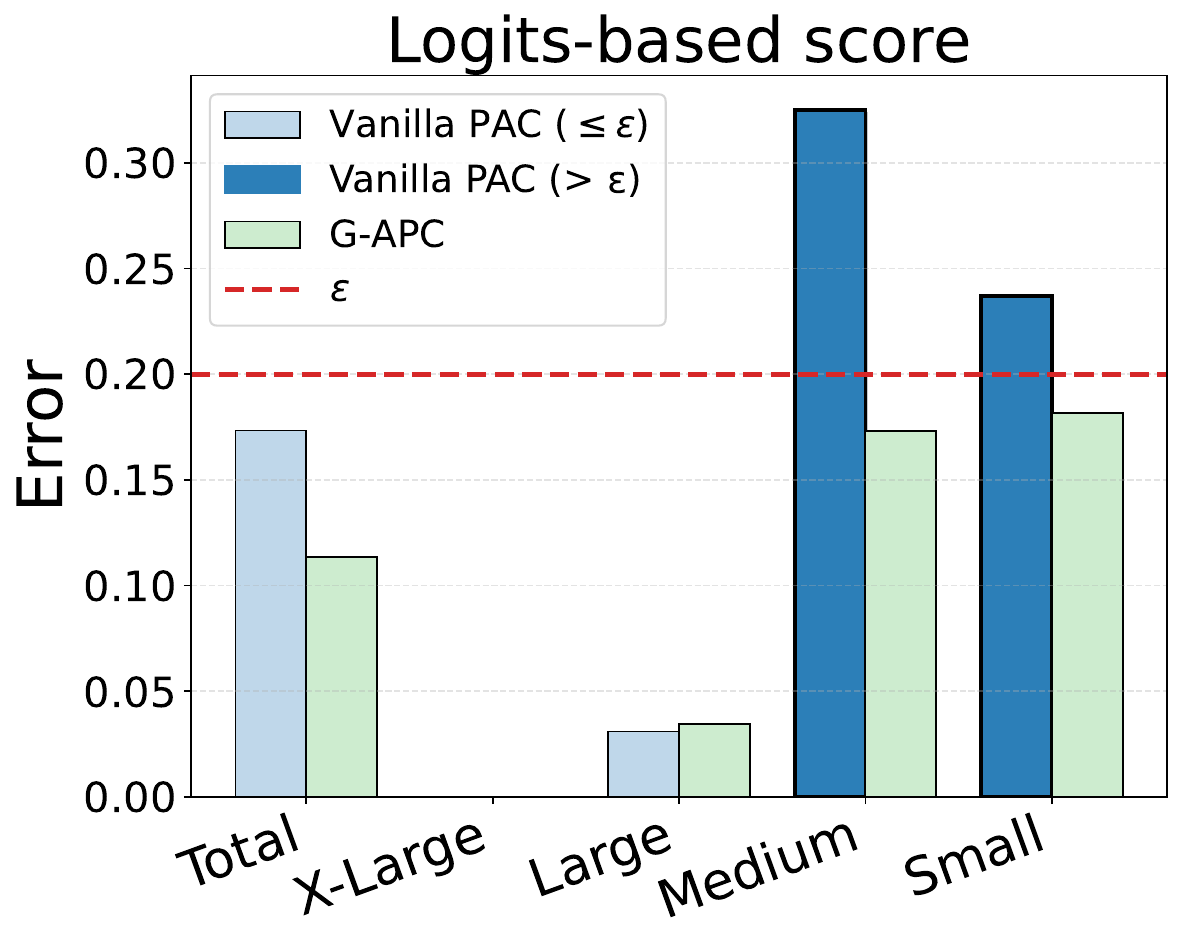}
        \includegraphics[width=0.27\linewidth]{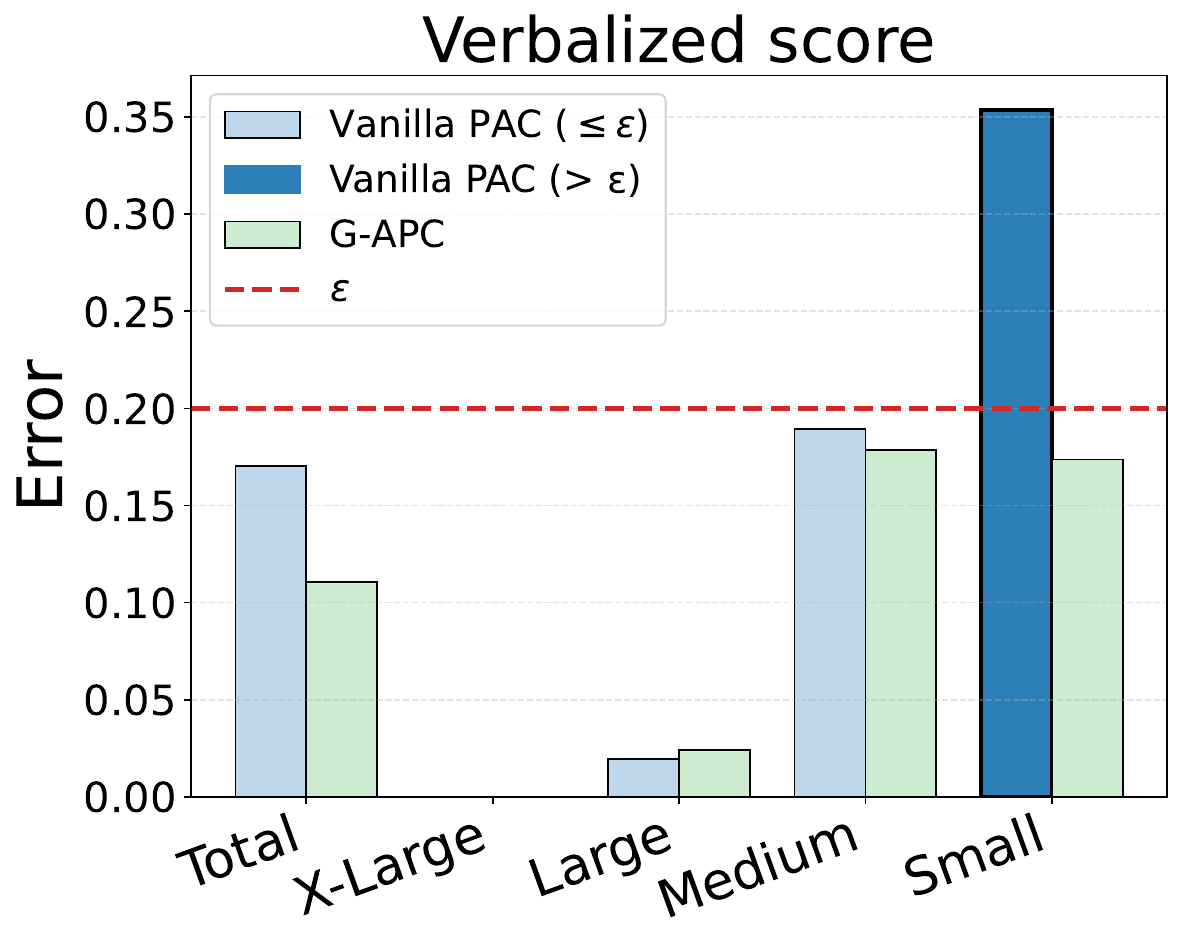}
        \includegraphics[width=0.27\linewidth]{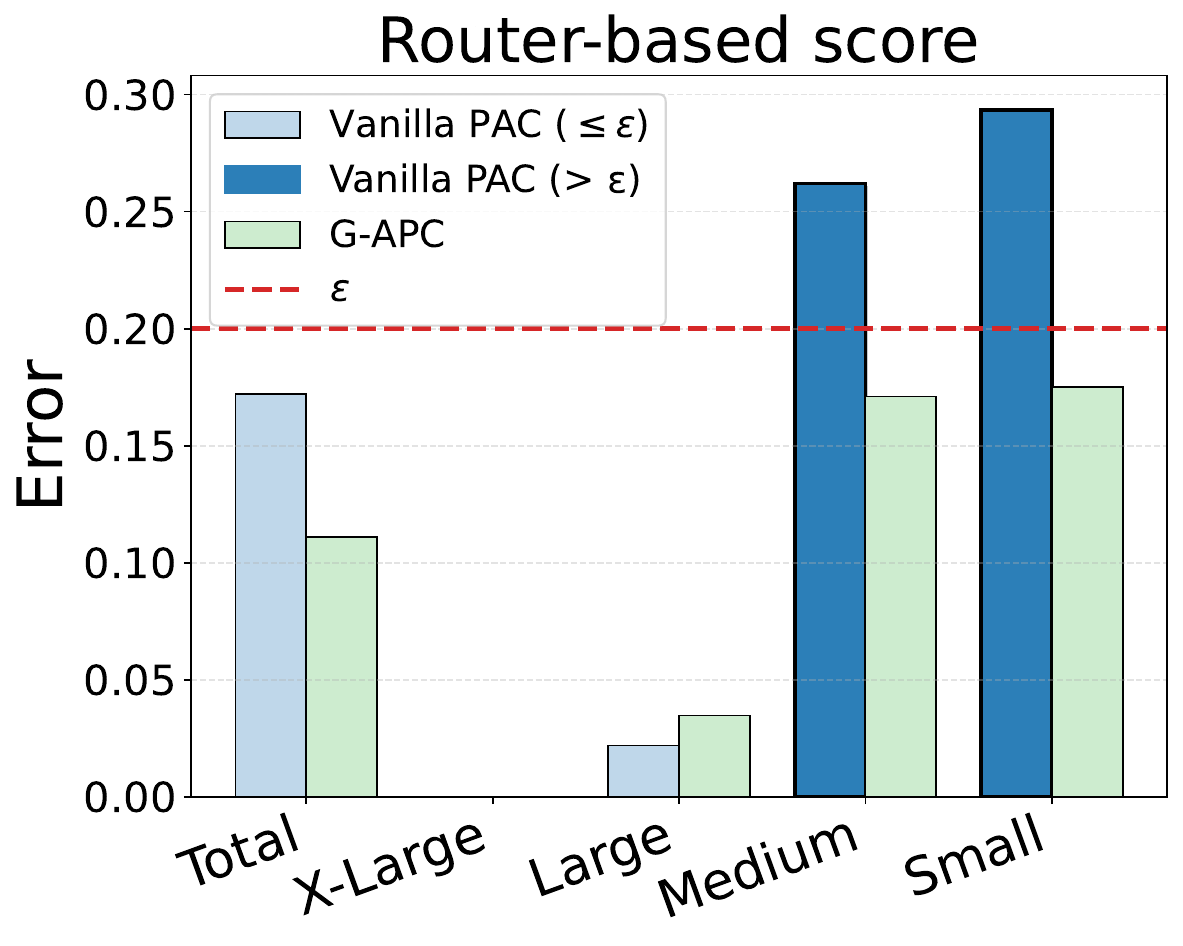}
        \caption*{ZebraLogic}
    \end{subfigure}
    \begin{subfigure}[t]{\linewidth}
        \centering
        \includegraphics[width=0.27\linewidth]{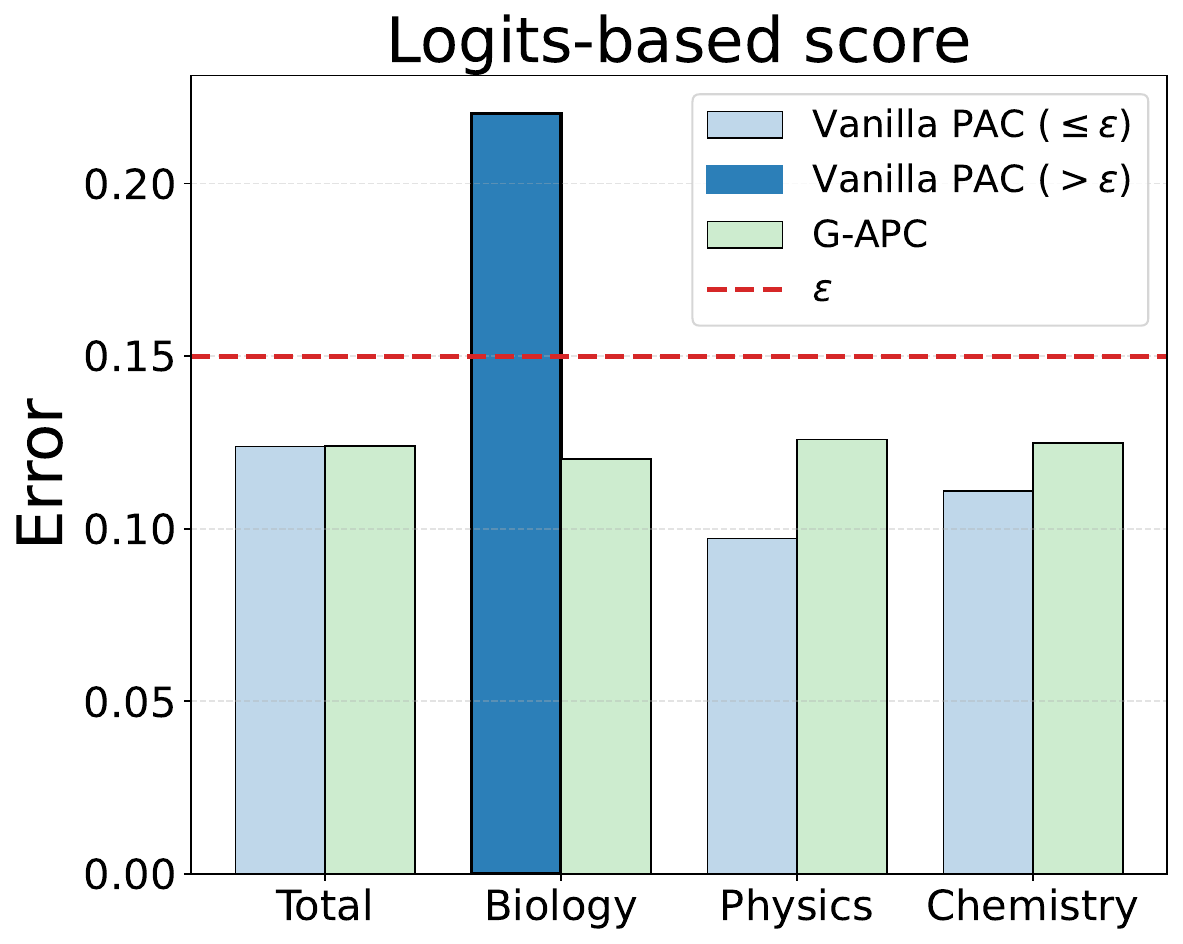}
        \includegraphics[width=0.27\linewidth]{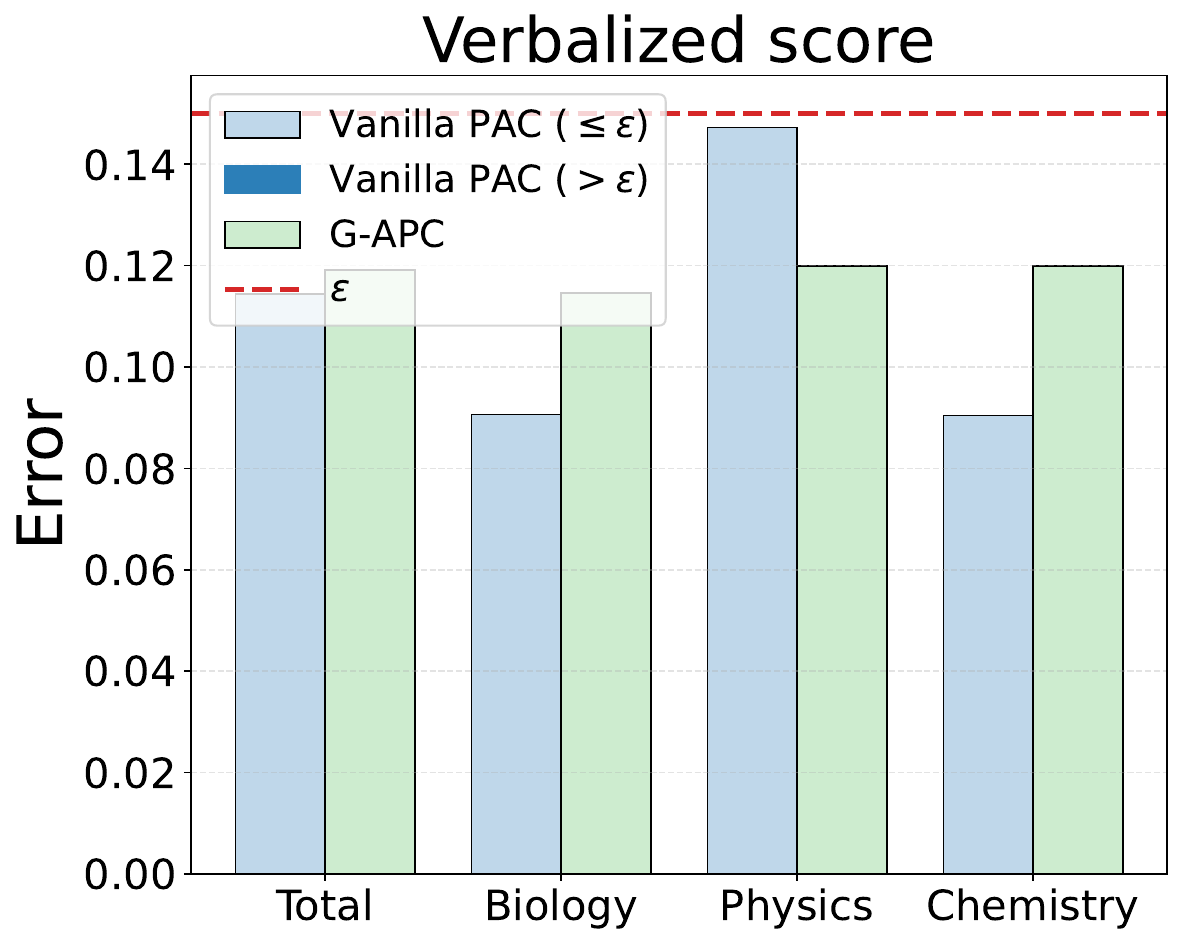}
        \includegraphics[width=0.27\linewidth]{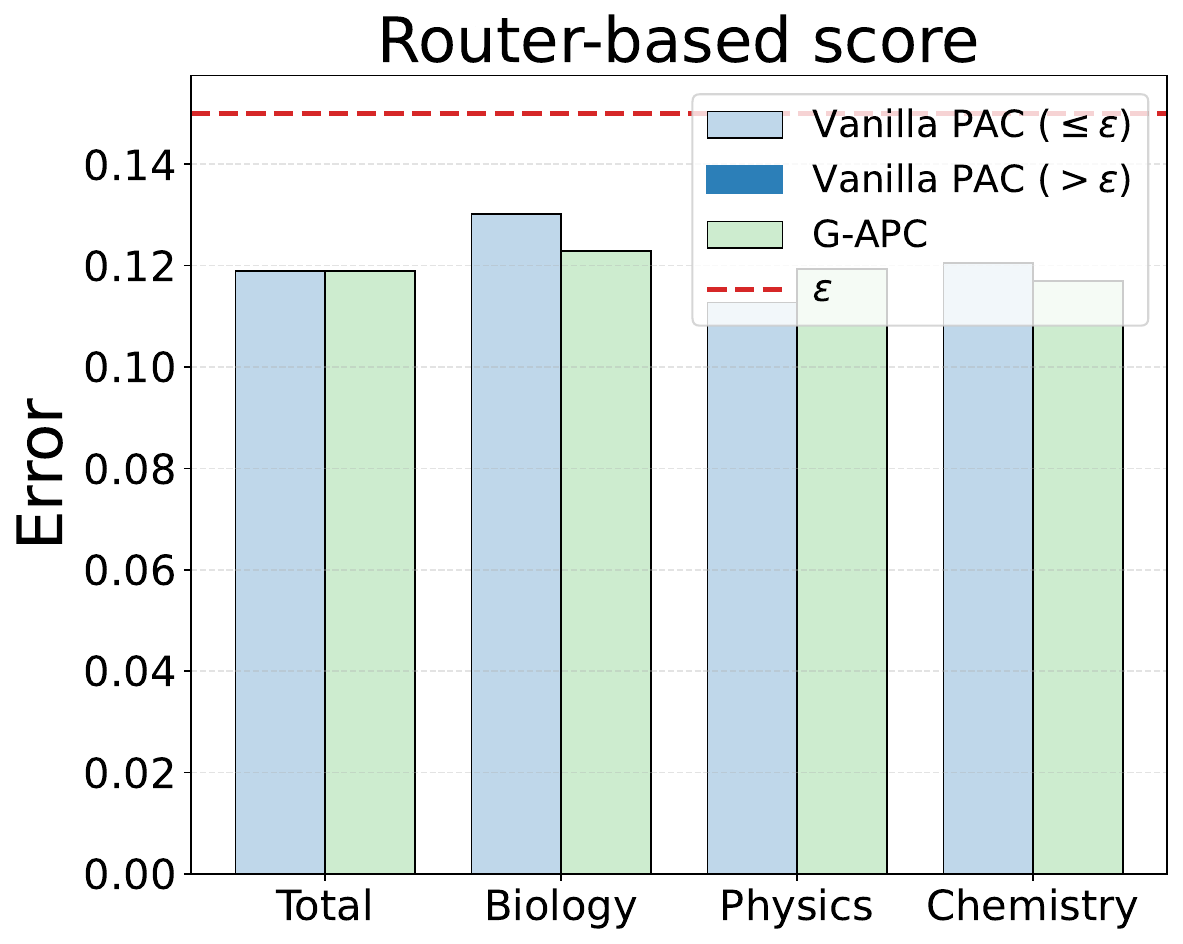}
        \caption*{GPQA}
    \end{subfigure}
    \caption{
   \textbf{G-PAC controls the group-conditional performance loss below the target while vanilla PAC fails (dark blue) across three different uncertainty scores.}
   Experiments are conducted with the Llama-based model.
   The figure reports the overall and per-category performance losses on three reasoning benchmarks.}
    \label{fig:errorbar_category_llama}
\end{figure*}

\begin{table*}[!t]
\centering
\caption{\textbf{Performance comparison between vanilla PAC and G-PAC under known group partitions.}  
Experimental results of the binary loss function on verifiable datasets using the Qwen models with $\alpha=0.05$.  
We set $\varepsilon=0.4$ for MATH-500, $\varepsilon=0.2$ for ZebraLogic, and $\varepsilon=0.15$ for GPQA.}
\label{table:01_results_unknown_llama}
\setlength{\tabcolsep}{2mm}{
\resizebox{\textwidth}{!}{
\begin{tabular}{lccccccccccccc}
\toprule
\multirow{3}{*}{Dataset}  & \multirow{3}{*}{Metric}  & \multicolumn{6}{c}{Disjoint clustering}& \multicolumn{6}{c}{Joint clustering} \\
\cmidrule(r){3-8} \cmidrule(r){9-14}
& & \multicolumn{2}{c}{Logits-based score} &\multicolumn{2}{c}{Verbalized score} & \multicolumn{2}{c}{Router-based score} & \multicolumn{2}{c}{Logits-based score} &\multicolumn{2}{c}{Verbalized score} & \multicolumn{2}{c}{Router-based score}
\\
\cmidrule(r){3-4} \cmidrule(r){5-6} \cmidrule(r){7-8} \cmidrule(r){9-10} \cmidrule(r){11-12} \cmidrule(r){13-14}
& & PAC &C-PAC& PAC & C-PAC & PAC & C-PAC  & PAC & C-PAC& PAC & C-PAC & PAC & C-PAC \\
\midrule

\multirow{3}{*}{MATH-500 }
& $\text{Error}$ (\%)
& 33.22 & 33.78 & 33.38 & 31.64 & 33.54 & 34.41
& 35.08 & 34.06 & 34.86 & 32.37 & 35.34 & 34.83 \\

& $\text{Error}_{\text{Gap}}$ (\%)
& 8.22 & \colorbox{green!10}{0.00} & 7.68 & \colorbox{green!10}{0.00} & 2.16 & \colorbox{green!10}{0.00}
& 8.36 & \colorbox{green!10}{0.00} & 8.90 & \colorbox{green!10}{0.00} & 4.07 & \colorbox{green!10}{0.00} \\

& STP (\%) $\uparrow$
& 3.29 & 8.40 & 3.51 & 0.40 & 24.28 & 22.63
& 7.99 & 8.96 & 7.91 & 0.96 & 26.47 & 23.59 \\

\midrule

\multirow{3}{*}{ZebraLogic}
& $\text{Error}$ (\%)
& 16.53 & 15.21 & 16.26 & 12.32 & 16.05 & 16.38
& 16.82 & 15.88 & 16.15 & 12.27 & 16.24 & 16.57 \\

& $\text{Error}_{\text{Gap}}$ (\%)
& 3.10 & \colorbox{green!10}{0.00} & 13.62 & \colorbox{green!10}{0.00} & 3.74 & \colorbox{green!10}{0.00}
& 3.70 & \colorbox{green!10}{0.00} & 13.43 & \colorbox{green!10}{0.00} & 4.20 & \colorbox{green!10}{0.00} \\

& STP (\%) $\uparrow$
& 21.81 & 13.31 & -30.62 & -31.89 & 23.33 & 22.65
& 22.75 & 17.03 & -32.76 & -32.38 & 23.75 & 23.23 \\

\midrule

\multirow{3}{*}{GPQA}
& $\text{Error}$ (\%)
& 11.96 & 11.98 & 12.14 & 12.11 & 12.29 & 11.84
& 12.65 & 11.92 & 12.02 & 11.77 & 12.34 & 11.92 \\

& $\text{Error}_{\text{Gap}}$ (\%)
& 10.75 & \colorbox{green!10}{0.00} & 5.66 & \colorbox{green!10}{0.00} & 10.47 & \colorbox{green!10}{0.00}
& 12.32 & \colorbox{green!10}{0.00} & 5.53 & \colorbox{green!10}{0.00} & 11.05 & \colorbox{green!10}{0.00} \\

& STP (\%) $\uparrow$
& 16.75 & 32.73 & 38.86 & 34.85 & 40.77 & 37.83
& 18.89 & 31.53 & 38.51 & 31.36 & 41.27 & 37.19 \\

\bottomrule
\end{tabular}}}
\end{table*}

\begin{table}[!t]
\centering
\caption{\textbf{C-PAC successfully controls the group-conditional risk on an open-domain task.}  
Experimental results on Arena-Hard with $\varepsilon=0.15$, using Llama-based LLMs.}
\label{table:arena_results_llama}
\small
\setlength{\tabcolsep}{6mm}
\resizebox{0.6\columnwidth}{!}{
\begin{tabular}{ccccc}
\toprule
Uncertainty & Method & Error (\%) & Error$_{\text{Gap}}$ (\%) & STP (\%) $\uparrow$ \\
\midrule
\multicolumn{5}{c}{Disjoint clustering} \\
\midrule
\multirow{2}{*}{Logits}
 & PAC   & 13.51 & 3.88 & 42.70 \\
 & C-PAC & 13.08 & \colorbox{green!10}{0.00} & 36.63 \\
\midrule
\multirow{2}{*}{Verbalized}
 & PAC   & 13.34 & 6.35 & 29.49 \\
 & C-PAC & 12.61 & \colorbox{green!10}{0.00} & 22.50 \\
\midrule
\multirow{2}{*}{Router}
 & PAC   & 13.16 & 8.67 & 13.00 \\
 & C-PAC & 13.18 & \colorbox{green!10}{0.00} & 36.32 \\
\midrule
\multicolumn{5}{c}{Joint clustering} \\
\midrule
\multirow{2}{*}{Logits}
 & PAC   & 13.46 & 3.91 & 42.62 \\
 & C-PAC & 13.32 & \colorbox{green!10}{0.00} & 38.05 \\
\midrule
\multirow{2}{*}{Verbalized}
 & PAC   & 13.23 & 7.10 & 28.95 \\
 & C-PAC & 12.70 & \colorbox{green!10}{0.00} & 23.06 \\
\midrule
\multirow{2}{*}{Router}
 & PAC   & 13.30 & 9.38 & 13.49 \\
 & C-PAC & 13.20 & \colorbox{green!10}{0.00} & 35.96 \\
\bottomrule
\end{tabular}}
\end{table}

\end{document}